\documentclass{article} % For LaTeX2e
\usepackage[preprint]{colm2026_conference}

\usepackage{microtype}
\usepackage{hyperref}
\usepackage{url}
\usepackage{booktabs}
\usepackage{subcaption}
\usepackage{enumitem}

\usepackage{lineno}

\definecolor{darkblue}{rgb}{0, 0, 0.5}
\hypersetup{colorlinks=true, citecolor=darkblue, linkcolor=darkblue, urlcolor=darkblue}

\usepackage{graphicx} % Required for inserting images
\usepackage[utf8]{inputenc} % allow utf-8 input
\usepackage[T1]{fontenc}    % use 8-bit T1 fonts
\usepackage{url}            % simple URL typesetting
\usepackage{booktabs}       % professional-quality tables
\usepackage{amsfonts}       % blackboard math symbols
\usepackage{nicefrac}       % compact symbols for 1/2, etc.
\usepackage{microtype}      % microtypography
\usepackage{xcolor}         % colors
\usepackage{bm}

\usepackage[makeroom]{cancel}
\usepackage{amsmath}
\usepackage{amsthm}
\usepackage{mathtools}
\usepackage{wrapfig}

\usepackage[capitalize]{cleveref}
\crefname{appendix}{appendix}{appendices}
\Crefname{appendix}{Appendix}{Appendices}

\newtheorem{lemma}{Lemma}

\newtheorem{proposition}{Proposition}

\newcommand{\bX}{\bm{X}}

\newcommand{\bx}{\bm{x}}
\newcommand{\logit}{\bm{g}}

\newcommand{\bH}{\bm{{H}}}
\newcommand{\bh}{\bm{{h}}}
\newcommand{\bY}{\bm{{Y}}}

\newcommand{\R}{\mathbb{R}}

\newcommand{\Router}{P}
\newcommand{\router}{r}

\newcommand{\norm}[1]{\left\lVert#1\right\rVert}

\newcommand{\diag}[1]{\mathrm{diag}\left(#1\right)}

\newcommand{\one}{\mathbf{1}}
\newcommand{\PiE}{\Pi_E}
\newcommand{\softmax}{\operatorname{softmax}}

\usepackage[most]{tcolorbox}
\newcommand{\bp}{\bm{p}}

\title{
The Myth of Expert Specialization in MoEs: 
Why Routing Reflects Geometry,
Not Necessarily Domain Expertise. 
}

\author{Xi Wang \\
Johns Hopkins University\\
\texttt{xwang457@cs.jhu.edu} \\
\And
Soufiane Hayou \\
Johns Hopkins University \\
\texttt{hayou@jhu.edu} \\
\AND
Eric Nalisnick \\
Johns Hopkins University\\
\texttt{nalisnick@jhu.edu}
}

\usepackage{listings}
\usepackage{xcolor}

\begin{document}

\ifcolmsubmission
\linenumbers
\fi

\maketitle

\begin{abstract}
Mixture of Experts (MoEs) are now ubiquitous in large language models, yet the mechanisms behind their "expert specialization" remain poorly understood. We show that, since MoE routers are linear maps, hidden state similarity is both necessary and sufficient to explain expert usage similarity, and specialization is therefore an emergent property of the representation space, not of the routing architecture itself.
We confirm this at both token and sequence level across five pre-trained models.
We additionally prove that load-balancing loss suppresses shared hidden state directions to maintain routing diversity, which might provide a theoretical explanation for specialization collapse under less diverse data, e.g. small batch.
Despite this clean mechanistic account, we find that specialization patterns in pre-trained MoEs resist human interpretation: expert overlap between different models answering the same question is no higher than between entirely different questions ($\sim$60\%); prompt-level routing does not predict rollout-level routing; and deeper layers exhibit near-identical expert activation across semantically unrelated inputs, especially in reasoning models.
We conclude that, while the efficiency perspective of MoEs is well understood, understanding expert specialization is at least as hard as understanding LLM hidden state geometry, a long-standing open problem in the literature.
\end{abstract}

\section{Introduction}

Mixture of Experts \citep[MoEs,][]{jacobs1988increased, fedus2022switch} have become a standard architecture for large language models (LLMs). Compared with decoder-only LLMs~\citep{radford2018improving}, MoEs replace the dense feed-forward network (FFN) at each transformer block with a sparse mixture of expert blocks, where each token activates only a sparse subset of FFNs, keeping per-token compute manageable while scaling total model capacity.

An \emph{expert specialization} behavior has been widely observed in pre-trained MoEs: sequences from different domains tend to activate markedly different expert subsets (e.g., Sec.~5 in \citet{muennighoff2024olmoe}, Fig.~9 in \citet{liu2024deepseek}). Yet despite its importance as both a signal of meaningful feature learning and a prerequisite for parameter efficiency, its mechanistic origins remain largely unexplained. This paper asks: (1) \emph{what mechanism produces expert specialization?} and (2) \emph{to what extent are the resulting patterns humanly interpretable?}

% we perform a detailed study of the interactions between the router and hidden states across 5 different pre-trained models released by different labs

The key observation is simple: the MoE router is a linear projection from token hidden states to expert logits, therefore tokens with similar hidden states must activate similar experts. We formalize this as Prop.~\ref{prop:embed_logit_bound}, a tight data-geometry-aware upper bound on logit distance, verified empirically across five pre-trained MoEs (Fig.~\ref{fig:rms}). Specialization is therefore a consequence of how the model organizes its representation space, not a property of the routing mechanism itself. We further show (Prop.~\ref{prop:lbl_correlation}) that when using a load-balancing loss to train MoEs, the router is forced to ``suppress'' shared directions in correlated hidden states, which might provide a theoretical explanation to the specialization collapse observed by \citet{qiu2025demons} under small-batch training.

Despite this clean mechanistic picture, specialization patterns frequently defy human intuition. We demonstrate three cases: (1) different models solving the same math question activate expert patterns as dissimilar as those for entirely different questions ($\sim$60\% overlap, Fig.~\ref{fig:math_arena});
(2) semantically different prompts activate nearly identical experts during prefilling, but diverge during generation (Fig.~\ref{fig:rollout_vs_expert_usage});
(3) in deeper layers, sequences from unrelated domains can collapse to identical expert usage during prefilling (Figs.\ref{fig:hle_vs_wlmbledon},~\ref{fig:duplication_gpt_ling}), yet the suppressed experts remain consequential once generation begins. 

Together, these findings reframe expert specialization as a consequence of the geometry of LLM hidden states, rather than as a standalone interpretable phenomenon. Understanding it is therefore at least as difficult as understanding geometry, which remains an open question even in dense models \citep{lad2024remarkable,skean2025layer}. It may even be harder, since tokens with different hidden states can still activate similar experts.

\section{Related work}
\paragraph{Identification of domain-specialized experts} Several works study expert specialization from an interpretability perspective, aiming to verify its existence and locate domain-specific experts in pre-trained MoEs.
\citet{chaudhari2025moe} introduces a LogitLens-based framework for analyzing MoE specialization patterns, \citet{hu2026gets} proposes a causal-effect method to identify which experts are domain-specific while \citet{wang2026illusion} propose a CommitteeAudit framework. Overall, these works share similar observations: There exists a compact, domain-invariant expert that captures the majority of routing mass across all domains, with some peripheral experts showing domain sensitivity.
Our findings offer a mechanistic explanation for this pattern: shared committee experts likely reflect shared information in the hidden state space, potentially arising from shared syntax structure in the datasets they studied or simply from LLM internal computation.
Our work differs from these in two key ways. First, rather than cataloguing \emph{which} experts are domain-specific, we identify \emph{why} expert usage may or may not correlate with domain / semantic at all, grounding the phenomenon in hidden state geometry. Second, these works rely on short benchmark inputs (MMLU: $<$100 tokens; GSM8K: $<$200 tokens), whereas we study much longer sequences, including multi-thousand-token reasoning chains, where the gap between prefilling and generation behavior becomes critical.

\paragraph{Applications of expert specialization} Beyond interpretability, expert specialization has found practical applications in several domains. In multilingual settings, \citet{chen2026understanding} and \citet{bandarkar2026multilingual} leverage routing patterns to identify language-specific experts and improve MoE performance on low-resource languages. From a different angle, recent work in the security community~\citep{ding2025moecho, nuriyev2026expert} independently arrives at a perspective closely related to ours: because expert usage patterns are a deterministic function of input hidden states, they constitute a side channel that can leak information about user queries, raising significant privacy concerns. 
% This line of work underscores that the router–hidden-state relationship we study is not merely of theoretical interest — it has direct implications for both capability and safety in deployed systems.

% Extended related work on expert specialization discussion from pre-training reports and layer-wise hidden state behavior in dense models can be found in App.~\ref{app:extended_related_work}.

% \section{Extended related work}\label{app:extended_related_work}
\paragraph{Expert specialization insights from pre-trained MoE reports} Many open-sourced MoEs' reports have discussed expert specialization, typically by comparing expert activation patterns across corpora drawn from different domains, e.g., Appendix C in DeepSeek V3~\citep{liu2024deepseek}, Sec 6 in dots.llm1~\citep{huo2025dots}, and Figure 22 in OLMoE~\citep{muennighoff2024olmoe}. 
% Reports from  and OpenMoE~\citep{xue2024openmoe} have more in-depth analysis and the most detailed discussion on expert specialization, where they further looked into the vocabulary level expert activation pattern.
A notable exception is Mistral~\citep{jiang2024mixtral}, whose experts show little specialization, likely a consequence of sparse upcycling during training.
% Work from Qwen's team~\citep{qiu2025demons} demonstrates a failure pattern of MoE lacking specialization: When load balancing loss is enforced on a small batch of data with limited diversity
% Qwen's recent work~\citep{li2026expert} further propose to 
It is worth noting that the goal of monitoring specialization during pre-training is primarily practical rather than interpretive: the goal is to prevent expert homogenization, i.e. the failure mode in which all experts converge to redundant representations, wasting model capacity. DeepSeek V3~\citep{liu2024deepseek} operationalizes this by studying the sensitivity of removing top experts, arguing that MoE with better specialization shows higher sensitivity to expert removal, indicating less redundancy; 
% ERNIE~\citep{ernie2025technicalreport} includes Orthogonalization loss on the router...

\paragraph{Layer-wise behavior of hidden states in dense models}
A key implication of our work is that expert specialization can vary across depth because hidden state geometry evolves across layers. A growing body of work studies this evolution in dense transformers.
\citet{lad2024remarkable} identifies four functional stages of inference — detokenization, feature engineering, prediction ensembling, and residual sharpening — to explain non-uniform performance degradation under layer removal. \citet{skean2025layer} shows that intermediate layers consistently produce richer representations than final layers for downstream tasks. \citet{acevedo2026differential} further finds that syntactic and semantic information are encoded differently across layers. These results translate directly to expert specialization: what a layer's hidden states encode determines which experts it activates, explaining why specialization patterns differ between early and late layers.

\section{Background: MoE}

\paragraph{Decoder only LLM} For a decoder only LLM with $L$ layers, let $\bX^\ell = \{\bx_t^\ell\}_{t=1}^{T}\in \R^{T \times D}$ denote the input to layer $\ell$, where $T$ is the sequence length and $D$ is the residual-stream dimension. A standard transformer block first applies self-attention and then an MLP sublayer, each wrapped with layer normalization and a residual connection:
\begin{align}
    & \bH^\ell = \bX^\ell + \text{SelfAttention}(\text{LN}(\bX^\ell)),
     & \bY^\ell = \bH^\ell + \text{FFN}(\text{LN}(\bH^\ell)),
\end{align}
where we assume a pre-norm formulation, $\text{LN}(\cdot)$ denotes layernorm, $\bH^\ell$ is the intermediate hidden representation after attention and $\bY^\ell$ is the output of layer $\ell$.

\paragraph{Token-choice sparse MoE} Recent models replace the dense MLP with a sparse mixture-of-experts (MoE) module, so that only a small subset of parameters is activated for each token while maintaining the total model capacity unchanged.
Concretely, for a token representation $\bh_t^\ell \in \R^D$, i.e. a row from $\text{LN}(\bH^\ell)$, the MoE router computes routing logits
\begin{equation}\label{eq:router_def}
    \logit_t^\ell = \Router^\ell \bh_t^\ell \in \R^{E},
\end{equation}
where $\Router^\ell \in \R^{E \times D}$ is the router weight matrix and $E$ is the number of experts. Based on these logits, the router selects a small subset $\mathcal{S}_t^\ell \subseteq \{1,\dots,E\}$ of size $k$, typically the top-$k$ experts with the largest logit value. 
Let $s_{t,e}^\ell = \sigma(\logit_t^\ell)_e $ denote the routing weight assigned to expert $e$, usually obtained by setting $\sigma(\cdot)$ to be a softmax or sigmoid-and-normalize.
The MoE output is then computed as
\begin{equation}
    \text{MoE}(\bh_t^\ell)
    = \Sigma_{e \in \mathcal{S}_t^\ell} s_{t,e}^\ell \, \text{FFN}_e^\ell(\bh_t^\ell),
\end{equation}
where $\text{FFN}_e^\ell(\cdot)$ is the feed-forward network corresponding to expert $e$. Thus, different tokens may be processed by different fixed-size subsets of experts, enabling the model to increase parameter count while keeping per-token computation relatively small.
This token-choice routing formulation covers the most recent sparse MoE models. Variations of model structures usually occur in including extra shared experts, whether a jitter noise is added to $\bh$ before routing, different sparsity ratios $\frac{k}{E}$, or the design of $\sigma$.

% An important exception is Llama 4, which uses \emph{expert-choice} routing, where experts select tokens rather than tokens selecting experts. Despite this difference, both paradigms rely on routing scores that determine how token representations are assigned to experts.

\section{Expert ``specialization'' arises from hidden state similarity}
%Expert specialization arises from hidden state similarity

% Historically MoEs, as its name suggests, have each sub-module responsible for different functionality. 

% MLPs in dense transformer is a large non-linear transformation.Routers in MoE, on the otherhand, are much easier to understand. They are just linear projections, so the patterns you see in the expert usage reflects the patterns in the embeddings.

\subsection{Hidden states, expert usage, and the role of auxiliary loss}
To understand the expert activation pattern, we begin by looking at the relationship between the tokens' hidden state ($\bh$) and corresponding router outputs ($\logit$). 
The idea is simple: since the router Eq.~\eqref{eq:router_def} is a linear projection, we expect tokens with similar $\bh$ to show similar $\logit$. To formalize this intuition, we measure the similarity between the hidden states using the Euclidean distance in a data-dependent space, and obtain a tight upperbound on how router outputs behave as a function of hidden states.
\begin{proposition}\label{prop:embed_logit_bound}
Let \(H \in \mathbb{R}^{N\times D}\) be a data matrix whose rows are \(\bh_1^\top,\dots,\bh_N^\top\), and let
\(H = U \Sigma V^\top\) be its singular value decomposition. Denote by $m=\min\{N,D\}$ the rank of $H$. Fix \(r \in \{1,2, \dots, m\}\), and let \(V_r \in \mathbb{R}^{D\times r}\) be the matrix formed by the top \(r\) right singular vectors of \(H\). Define the orthogonal projector onto the corresponding principal subspace by $\Pi_r := V_r V_r^\top \in \mathbb{R}^{D\times D}.$
Let $P \in \mathbb{R}^{E\times D}$ be the router weight matrix. Then, for any $i,j \in \{1,\dots,N\}$,
\begin{equation}\label{eq:norm_bound}
\underbrace{\|P\bh_i - P\bh_j\|_2}_{\text{Logits similarity}}
\le
\underbrace{\|P\Pi_r\|_2}_{\text{Router-data alignment}} \, \underbrace{\|\bh_i-\bh_j\|_{\Pi_r}}_{\text{Hidden state similarity}}
+
\underbrace{\|P(I-\Pi_r)(\bh_i-\bh_j)\|_2}_{Residual}.
\end{equation}
% In particular,
% \[
% \|Px_i - Px_j\|_2
% \le
% \|P\Pi_r\|_2 \, \|x_i-x_j\|_2
% +
% \|P(I-\Pi_r)\|_2 \, \|(I-\Pi_r)(x_i-x_j)\|_2,
% \]
% and
% \[
% \|(I-\Pi_r)(x_i-x_j)\|_2^2
% =
% \sum_{k=r+1}^{D} \langle x_i-x_j, v_k\rangle^2.
% \]
\end{proposition}
$\norm{\cdot}_2$ denotes operator norm (maximum singular value) for matrix input, and $\|z\|_{\Pi_r}=\|\Pi_r z\|_2$ denotes the norm in the space corresponding to the projector $\Pi_r$. The full proof is provided in App.~\ref{app:proof}, which first projects each $\bh_i$ on $\Pi_r$ and then applies the triangle inequality.

\begin{figure}[!t]
    \centering
    \includegraphics[width=.95\linewidth]{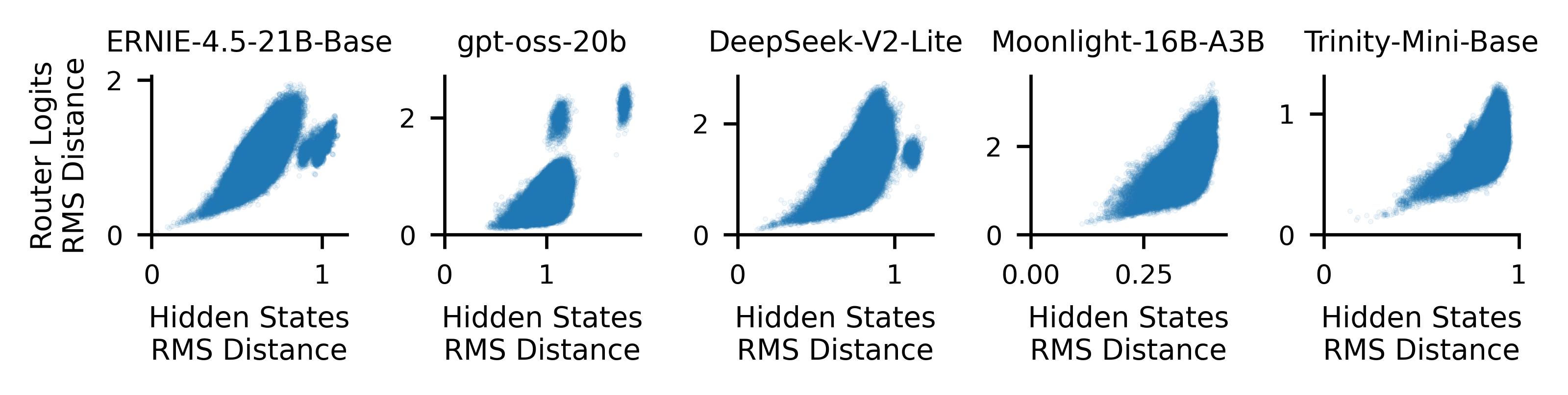}
    \\[-1.5ex]
    \caption{
    {\small{Empirical verification of Prop.~\ref{prop:embed_logit_bound}.
    On the middle layer of 5 models, using 16 sequences from OpenWebText, we compare the RMS distance (Eq.~\eqref{eq:rms_def}) of all token pairs' hidden state v.s. their router logits.
    \textbf{Bottom left corners:} Across all models, similar hidden states induce similar router logits, i.e. similar expert usage.
    \textbf{Middle right:} When the hidden states are different, the router logits can be similar or different.
    % Similiar token --> (provably) similiar expert usage, the other direction does not necessarily hold. Triangle shape tells the story. 
    \textbf{Outlier dots} are from distances to the massive activation of the first token, i.e.\ attention sink~\citep{sun2024massive}.}
    % Moonlight-16B shows much smaller hidden state distance because it is trained with Muon, where overall activation norms are more concentrated (Fig.~\ref{fig:muon_vs_adamw}).
    }}
    \label{fig:rms}
\end{figure}

\paragraph{Triangle shape of hidden states vs logit distances}
Under the assumption that the data variance is mostly concentrated in the top $r$ directions (right-most columns in Fig.~\ref{fig:alignment:gpt_oss_router_data_alignment},\ref{fig:alignment:trinity_router_data_alignment},\ref{fig:alignment:deepseek_router_data_alignment},\ref{fig:alignment:moonlight_router_data_alignment}), the residual term is negligible compared to the first term.
Proposition \ref{prop:embed_logit_bound} suggests that when two tokens' hidden states are similar in the sense of norm $\|.\|_{\Pi_r}$, the upper bound collapses; their corresponding logits, i.e., the activated experts, must be similar. This is a much sharper bound than the trivial bound $\|P\|_2 \|h_j - h_j\|_2$ (verified in Fig.~\ref{fig:bound_viz}, Appendix). However, when the hidden states differ, their router logits can be either similar or dissimilar. Fig.~\ref{fig:rms} empirically verifies this using OpenWebText on 5 different MoEs, where we observe the dots to demonstrate a \emph{triangle} shape with small / large hidden state distance regions being the triangle's angle / side. 

\paragraph{Impact of load-balancing} From \cref{prop:embed_logit_bound}, if two hidden states are highly correlated (and there norms are roughly of the same order), router outputs will be similar. This could happen for several reasons. One structural reason is depth: as depth grows, correlation between hidden states increases between token hidden states, leading to increasing similarity as we will see below. In the next result, we show that load-balancing loss helps mitigate this problem by forcing the router to ignore directions shared by hidden states to achieve balanced expert usage.

\begin{proposition}[Informal]\label{prop:lbl_correlation}
We consider a simple hidden state correlation model where the hidden states are represented by \vspace{-2pt}
    \[
    \bh_i = \mu + \xi_i, \quad i = 1, \ldots.N,
    \]
where \(\mu\in\R^D\) is a shared direction and $\xi_i$ are terms that differentiate each hidden state. We assume that $\xi_i$ are centered and i.i.d. Consider a simplified load balancing loss for router $\Router \in \R^{E\times D}$
\[
\mathcal{L}_{\text{bal}}(\Router) := \norm{\frac{1}{N}\sum_{i=1}^N {\softmax}\left({\Router \bh_i}\right) - \frac{1}{E}\mathbf{1}}^2.
\]
Then, a near-minimizer of $\mathcal{L}_{\text{bal}}(\Router)$ must satisfy
$\Router\mu \approx \overline{P\mu}\,\one, ~\text{where} ~~\overline{\Router\mu}:=\frac1E \one^\top \Router\mu$.
That is, the shared direction does not affect the routing as softmax and top-K are invariant to constant shift.
\end{proposition}
% in the appendix, where we show that minimizing the load balancing loss forces the shared direction in the data to be routing-insignificant.
The proof is presented in Appendix \ref{app:balance_loss_analysis}, where we assume the small-logit regime under which softmax can be linearized.
\cref{prop:lbl_correlation} shows that the load-balancing objective forces the router to suppress the shared direction $\mu$, which is expected since the shared direction pushes the hidden states to use similar experts. Realistically, load-balancing is \emph{added} to the pre-training objective, which means that we should only partially observe this effect. We believe this result could provide a theoretical explanation for the observation in \citet{qiu2025demons}, which suggests that when load-balancing is applied on small micro-batches (rather than the large global-batches), the resulting model loses any form of expert specialization, showing same expert usage patterns for inputs of different topics. Intuitively, small batches are not diverse enough, and therefore shared directions within small batches can significantly vary from one batch to another. This constant change in shared directions from one batch to another prevents the router from finding the informative directions. Similarly, Fig. 5 in \citet{wang2024auxiliary} showed that auxiliary loss-free approach works better with larger batch size. 
\begin{figure}[!t]
    \centering
    \begin{subfigure}{\linewidth}
    \centering
        \includegraphics[width=\linewidth]{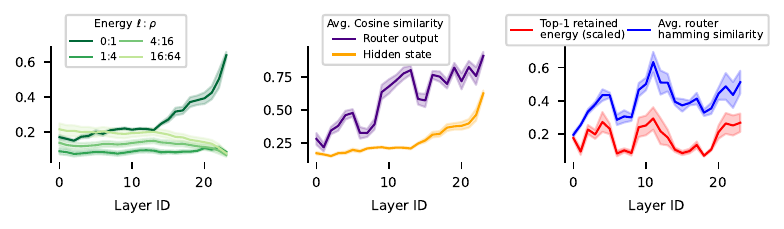}
        \caption{gpt-oss-20b, trained with AdamW and auxiliary load balancing loss.}
    \end{subfigure}
    \begin{subfigure}{\linewidth}
    \centering
        \includegraphics[width=\linewidth]{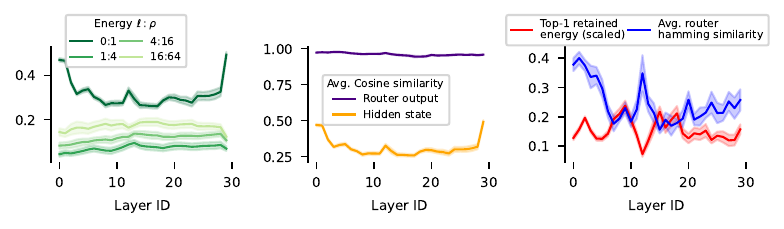}
        \caption{Trinity-Mini-Base, trained with Muon and auxiliary-free load balancing.}
    \end{subfigure}
     % \includegraphics[width=0.95\linewidth]{figures/correlation/arcee_router_inspect.pdf}
    % \\[-1.5ex]
    \caption{\small{
    On gpt-oss-20b and Trinity-Mini-Base, we studied tokens in 16 random sequences from OpenWebText.
    Lines / bars denote the mean / 1-std. over 16 sequences.
    Two models use different training configurations, resulting in distinct patterns across depth.
    \textbf{Top Left}: Token hidden states become more correlated with depth, indicated by stronger energy (Eq.~\eqref{eq:directional_energy}) from the shared direction (first right singular vector).
    \textbf{Top Middle}: After layer 11, the router output does not show a significant increase in cosine similarity, while hidden states do (Eq.~\eqref{eq:token_cos_def}).
    \textbf{Top Right}: Looking at the router, it learns to suppress the shared direction to maintain diversity of router outputs: The trace of router's projection (averaged over rows) onto the shared direction (red line, Eq.~\eqref{eq:retained_energy}) aligns with the shape of router output similarity (blue lines, Eq.~\eqref{eq:hamming_def})
    \textbf{Bottom left}: Under Muon trained model, token hidden states does not show increased correlation with depth.
    \textbf{Bottom middle and right}: When no load balancing loss is used, the router does not show evidence of suppressing shared direction.
    }}
    % (results for other models are shown in Fig.??? App ???}
    \label{fig:gpt_oss_correlation_depth}
\end{figure}\vspace{-5pt}

\paragraph{Auxiliary loss pushes router to de-correlate hidden states}
The upper bound in Eq.~\eqref{eq:norm_bound} involves the term $\|P \Pi_r\|$ which measures router $\leftrightarrow$ hidden-state alignment.
Intuitively, if this term vanishes, the logits would fail to capture the differences in the hidden states, effectively ignoring the data.
While we did not empirically verify this,\footnote{To rigorously verify this, a comprehensive set of experiments with MoE pretraining is required.} such a scenario could happen during MoE pre-training 
when the data is not diverse or the load balancing loss coefficient is set too large\footnote{Regularizing the router to orthogonal~\citep{omi2025load,ernie2025technicalreport} cannot alleviate this, as the router can be an orthogonal matrix with all rows orthogonal to the data principal directions.}. We did not observe such extreme ``dead router'' patterns in open-sourced MoEs, as these models often show suboptimal performance and would not be released to the public.
However, a related behavior can be observed in deeper layers of MoEs.
In particular, LLM's hidden states across tokens inside a sequence tend to show increasing similarity with depth (left Fig.~\ref{fig:gpt_oss_correlation_depth} for GPT-OSS, and Fig.~\ref{fig:token_sim_vs_depth_dense} for dense models), which effectively \emph{decreases} the diversity of expert usage. This is consistent with the \emph{rank-collapse} phenomenon that occurs with depth \citep{schoenholz2017deepinformationpropagation, hayou2019impactactivationfunctiondeep, hayou2021stableresnet, noci2022signalpropagationtransformerstheoretical}.
The router, driven by auxiliary load-balancing loss, would learn to suppress the shared direction (first singular vector) to maintain diverse router outputs. 
We confirm such suppression by comparing the router's retained energy on the shared direction with the expert usage similarity
(right-most Fig.~\ref{fig:gpt_oss_correlation_depth}, or by looking at the router -- singular-direction alignments, where the router shows less alignment with the first / shared direction compared with the second / third one (Middle columns, Fig.~\ref{fig:alignment:gpt_oss_router_data_alignment},\ref{fig:alignment:deepseek_router_data_alignment}).
In contrast, in models that are ``auxiliary loss free'', such as Moonlight 16B~\citep{liu2025muon} or Trinity Mini~\citep{singh2026arcee}, we do not observe such patterns (Fig.~\ref{fig:corr_appends_Moonlight},\ref{fig:corr_appends_Trinity},~\ref{fig:alignment:moonlight_router_data_alignment},~\ref{fig:alignment:trinity_router_data_alignment}).

% \sh{Not sure where the next two paragraphs were intended to fit.}\\
% we can veiw this from two perspectives:
% 1. If we track the lyaer wise retained energy v.s. the average expett usage hamming similaritry, we can see that the shape of the retained energy aligns with the similarity curve
% % Additionally, if we lookat the per layer router singular direction alignment, we notice that the first direction's alignment is often smaller than the second or the thrids ones, indciating that the model is surpressing the shared direcitni energy

% models exhibit less severe hidden state similarity, potentially due to their being both trained with Muon. 
% This also highlights that small batch sizes may not be the only reason for correlated hidden states; it is also affected by the optimizer and the depth scaling.

\begin{figure}[!t]
    \centering
    \includegraphics[width=0.37\linewidth]{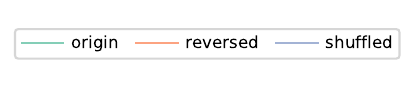}\\[-1ex]
    \centering
    \begin{subfigure}[t]{0.3\textwidth}
    \centering
     \includegraphics[width=\linewidth]{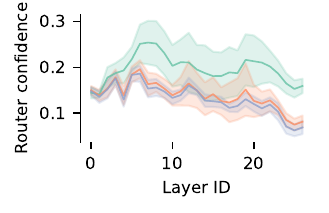}
     \caption{ERNIE-4.5-21B}
    \end{subfigure}
    \begin{subfigure}[t]{0.3\textwidth}
    \centering
     \includegraphics[width=\linewidth]{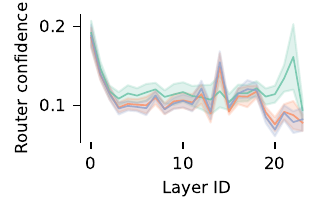}
     \caption{GPT-OSS-20B}
    \end{subfigure}
    \begin{subfigure}[t]{0.3\textwidth}
    \centering
     \includegraphics[width=\linewidth]{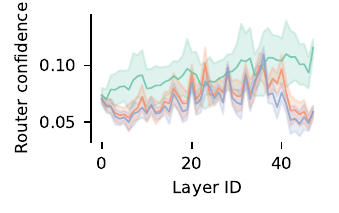}
     \caption{Qwen-3-30B}
    \end{subfigure}\\[-1.5ex]
    \caption{
    \small
    {
    We shuffled and reversed 16 sequences from OpenWebText to simulate OOD inputs.
    Router outputs show less confidence (measured by maximum softmax probability, Eq.~\eqref{eq:conf_def}) on OOD inputs compared with the original ones (green line), due to the reduced alignment data router alignment (Eq.~\eqref{eq:norm_bound}, verified in Fig.~\ref{fig:ood_mech}).
    % Error bars denote one standard deviation over 16 sequences.
    % OOD inputs show lower router confidence (maximum softmax probability). Error bars denote one standard deviation over 16 sequences. Mechanism explained in 
    }}
    \label{fig:ood_inputs}\vspace{-5pt}
\end{figure}

\paragraph{Router confidence detects OOD inputs}
Lastly, the router in pre-trained MoE is not a random matrix; it shows alignment with the data subspace during pre-training (left-most columns in Fig.~\ref{fig:alignment:gpt_oss_router_data_alignment},\ref{fig:alignment:trinity_router_data_alignment},\ref{fig:alignment:deepseek_router_data_alignment},\ref{fig:alignment:moonlight_router_data_alignment}).
This suggests that when the input sequence is out-of-distribution (OOD), its hidden states could show less alignment with the router,
reducing the upper bound and outputting more uniform logits.
We simulate OOD by reversing and shuffling tokens in the sequences.
Indeed, in Fig.~\ref{fig:ood_inputs} we see the confidence, measured by maximum softmax probability~\citep{hendrycks2016baseline}, is lower for OOD inputs compared with ordinary inputs. In Fig.~\ref{fig:ood_mech} in Appendix, we further verified the reduced alignment between router and reversed inputs compared with standard inputs.

\subsection{Expert activation patterns are context dependent}

In the previous section, we argued that for reasonably trained routers (non-collapsed), the token-wise expert usage is determined by its hidden state, whose value depends on: a) the token ID; b) the context.
In earlier layers, we expect the hidden state to be mainly dependent on the token ID as it is closer to the input embedding layer. In deeper layers, strong correlation between hidden states for different tokens appears because depth and self-attention induce more interaction with the other tokens, and therefore the hidden states and the expert usage become more context dependent.

We empirically verify expert usage's context dependency by putting the same token(s) under different contexts on \texttt{gpt-oss-20b}.  In Fig.~\ref{fig:same_seq_different_context} (top), we plot the activated experts for ``apple'' in daily context v.s. in coding context.
% and semicolon ``;'' as separation of sentence v.s. end of line in coding.
Broadly, in first layer, there exists a high overlap across context (more purple blocks); but in later layers, the overlap starts to reduce. Fig.~\ref{fig:same_seq_different_context} (bottom) performs a similar study but for a 6-token sentence in daily v.s. medical context, where we look at expert activation frequency for the 6 tokens.
Note that our observations go against the ``Context-Independent Specialization'' behavior observed by OpenMoe~\citep{xue2024openmoe}, which we believe arises from insufficient training.

\begin{wrapfigure}[24]{r}{0.35\textwidth}
\vspace{-3.2\baselineskip}
\captionsetup{justification=centering}
    \centering
    \includegraphics[width=.85\linewidth]{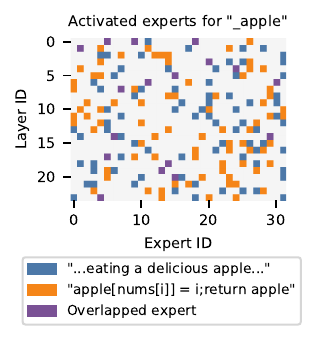}
    \includegraphics[width=\linewidth]{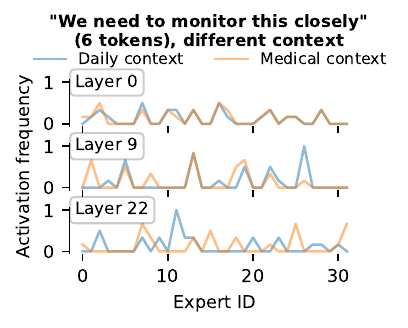}
\caption{\small{Same token (top) / sentence (bottom), different context, the expert usage is context dependent in deep layers.}}
\label{fig:same_seq_different_context}
% \vspace{-1em}
\end{wrapfigure}

% \begin{figure}[!t]
%     \centering
%     \includegraphics[width=0.325\linewidth]{figures/same_token_different_context_oss_apple_2.pdf}
%     \centering
%     ~\includegraphics[width=0.325\linewidth]{figures/same_token_different_context_oss_apple.pdf}
%         ~\includegraphics[width=0.325\linewidth]{figures/same_token_different_context_oss_sep.pdf}
%     \\[-1.5ex]
%     \caption{{
%     On GPT-OSS 20b, same token activates different experts (blue v.s. orange blocks).
%     High overlap (more purple blocks) in the first row near input layers.
%     Less overlap (less purple blocks) in intermediate layers.
%     First column (baseline / reference): ``apple'' in daily context;
%     Second column: ``apple'' in daily v.s. in coding context;
%     Last column: ``;'' as end of sentence v.s. end of line of code;
%     }}
%     \label{fig:token_vs_context}
% \end{figure}

\subsection{Sequence and dataset-level expert specialization}

\begin{figure}[!t]
    \centering
    \includegraphics[width=\linewidth]{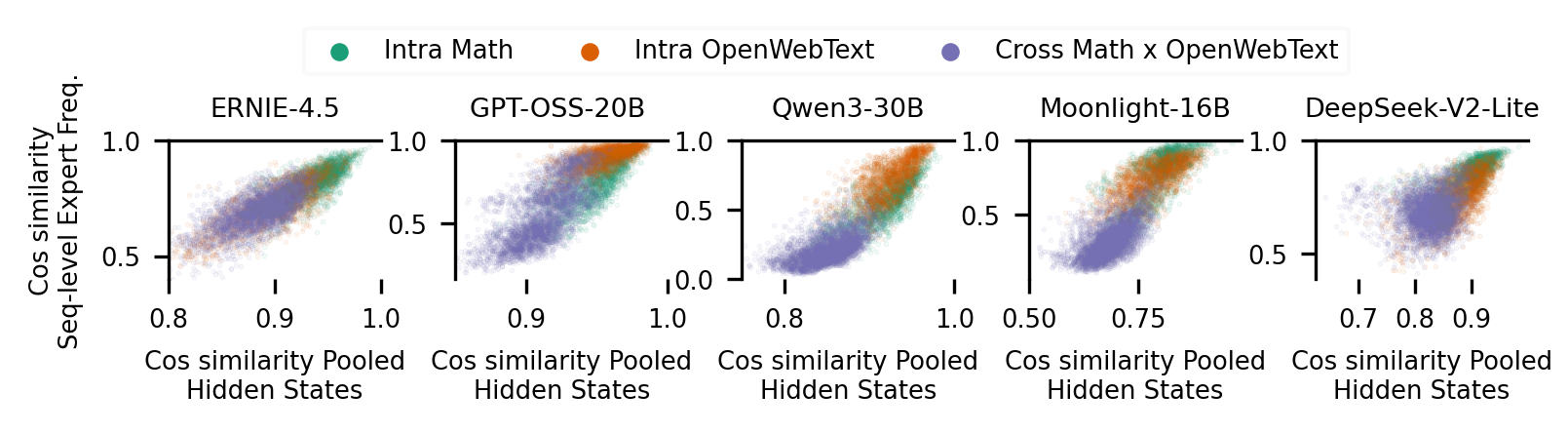}
    \\[-1.5ex]
    \caption{%
    \small{
        Pairwise comparisons of inputs drawn from Nemotron-CC-Math-v1, OpenWebText, and across these two datasets, plotting pooled hidden-state similarity (Eq.~\eqref{eq:seq_pooled_sim}) against expert-usage frequency similarity (Eq.~\eqref{eq:seq_expert_freq_sim}).
        Consistent with the token-level findings in Fig.~\ref{fig:rms}, more similar sequences tend to activate more similar experts.
    }}
    \label{fig:pooled_embd}
\end{figure}

The preceding analysis examined expert usage at the token level, showing that tokens with similar hidden states tend to route to similar experts.
We now extend this view to the sequence level by comparing the cosine similarity of pooled hidden states (Eq.~\eqref{eq:seq_pooled_sim}) against the cosine similarity of sequence-level expert activation frequencies (Eq.~\eqref{eq:seq_expert_freq_sim}).
Fig.~\ref{fig:pooled_embd} reports pairwise comparisons of 100 sequences from \texttt{OpenWebText} and 100 from \texttt{Nemotron-CC-Math}~\citep{karimi2025nemotroncc}, evaluated both within and across datasets at the middle layer of five models (full results appear in Fig.~\ref{fig:seq_emb_vs_expert_gpt},\ref{fig:seq_emb_vs_expert_qwen},\ref{fig:seq_emb_vs_expert_ernie} in Appendix).

Two patterns emerge. First, sequence similarity strongly correlates with expert-usage similarity across all models. Second, some models exhibit a triangular point cloud: while similar sequences consistently activate similar experts, dissimilar sequences can activate either similar \emph{or} different experts, a many-to-one relationship that already hints at interpretive 
difficulty. Notably, cross-dataset pairs cluster at lower hidden-state similarity than intra-dataset pairs, offering a straightforward explanation for the specialist behavior observed in prior work: inputs from different domains tend to have different hidden-state distributions, and therefore engage different sets of experts.

So far, the picture seems tractable: understand the hidden states, and expert usage follows.
But this view holds only at a coarse, dataset level.
As soon as we ask finer-grained questions, e.g. do two models solving the same problem use the same experts? does the prompt predict the rollout's routing 
behaviour? The tractability breaks down. This leads us to a key tension:
\begin{tcolorbox}[boxsep=1pt,left=2pt,right=2pt,top=0pt,bottom=0pt,colback=yellow!10, colframe=yellow!50!black]
Understanding expert specialization is at least as hard as understanding hidden states, and potentially harder, since inputs with \emph{different} hidden states can activate similar experts; inputs with similar hidden states can activate different experts once generation begins.
\end{tcolorbox}
The three examples in the following section each stress-test one facet of this tension: specialization across models, across rollout length, and across layers.
% \begin{figure}[!t]
% \captionsetup{justification=centering}
% \centering
%  \begin{subfigure}[t]{0.38\textwidth}
%     \centering
%      \includegraphics[width=\linewidth]{figures/same_token_different_context_gpt_oss_medical_vs_daily.pdf}
%      \caption{Same token,\\different context.}
%     \end{subfigure}
%  %    \begin{subfigure}[t]{0.32\textwidth}
%     \caption{Same sentence, different context \xw{Look at output} \xw{Look at last few tokens} Note that we expect the patterns here, especially ones in later layers, to be model dependent}
% \end{figure}

\begin{figure}[!t]
\vspace{-2mm}
    \centering
    \includegraphics[width=\linewidth]{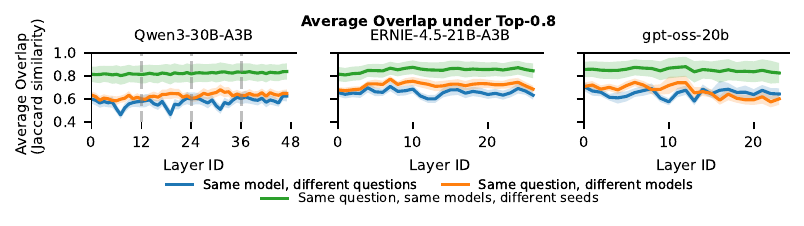}\\[-3.2ex]
    \caption{\small{
    Expert overlap (Jaccard similarity at Top-$p{=}0.8$, Eq.~\eqref{eq:overlap_at_p}) between rollout pairs 
    on 30 HMMT February 2025 questions, averaged over 4 seeds and three MoEs. 
    Different models solving the same question (orange) share only $\sim$60\% 
    of their most-used experts, on par with the same model solving different 
    questions (blue), and well below the same-question, same-model baseline 
    (green). The gap is consistent across all layers.
    }}
    \label{fig:math_arena}
\end{figure}

\begin{figure}[!t]
    \centering
    \includegraphics[width=\linewidth]{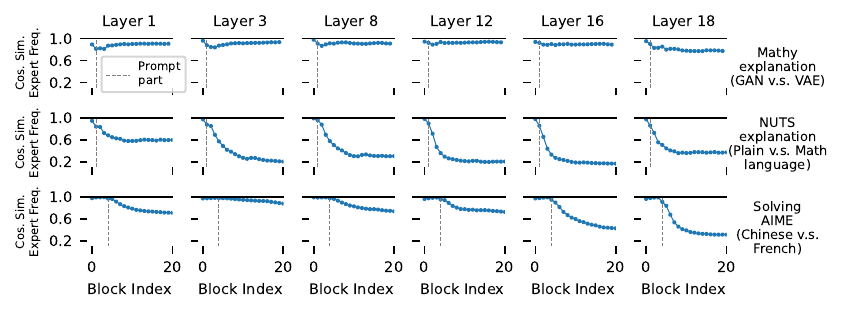}\\[-2ex]
    \caption{\small{
    Expert-usage similarity between paired queries on \texttt{Ling-mini}, 
    tracked from prompt through rollout. The dashed line 
    marks the prompt boundary. Across all three query pairs, prompt-phase 
    similarity is near-identical; once generation begins, similarity either 
    holds or drops sharply depending on the query and is not predictable 
    from the prompt alone.
    % Same math question, answered in different languages, y axis shows the similarity of expert usage frequency up to $k$ blocks of tokens. The similarity would gradually decrease throughout rollouts as the distribution of the two sequences gradually becomes diverse.
    }}
    \label{fig:rollout_vs_expert_usage}
\end{figure}

\vspace{-2pt}
\section{Why expert specialization patterns are hard to understand}
\label{sec:challenge}
\vspace{-2pt}
The analysis above might suggest that expert specialization is a stable, semantics-driven phenomenon that can be read off from the inputs. 
We show this fails at the level of individual problems, generations, and layers, and identify what drives each breakdown.

% \begin{itemize}[leftmargin=*, topsep=0pt, parsep=0pt, itemsep=1.5pt]
%     \item While it is easy to make macro-level arguments on entirely different datasets (e.g. Wiki v.s. Arxiv / Math v.s. Code) showing different hidden states, micro-level arguments are difficult to make, for example, given the same set of questions, are rollouts from different models similiar or not; Or whether the same solution written in different languages is similar or not.
%     \item  Behavior changes during rollout: In practice, we are interested in the experts usage of prompt + rollout, where the rollout can be very long for reasoning models. Similarity in the prompt does not imply similarity in rollouts
%     \item Different layers show different behaviors: While hidden states in earlier layers are easier to understand and dominated by the token distribution, the pattern in deeper layers may not depend on the input semantics but on the potential output patterns.
% \end{itemize}
% In the following sections, we will demonstrate three examples, one from each perspective. 

% The following three sections follow:
% - What we did
% - What we observed
% - What's the implication pattern

\subsection{Math solutions from different models activate different experts}

\citet{balunovic_srimatharena_2025} benchmarked a wide range of reasoning LLMs on math datasets, enabling us to study expert activation patterns across reasoning trajectories produced by different models on the same problems.
The natural expectation is that solutions to the same question should engage the same experts, regardless of which model generated them.

Our results contradict this expectation. We consider 30 questions from HMMT February 2025 solved by 60 (reasoning) models, and study all \texttt{(question, model, seed)} combinations via pairwise comparisons that fix two of the three axes. For each pair, we measure overlap at Top-$p$ (Eq.~\eqref{eq:overlap_at_p}): the fraction of shared experts responsible for $p\%$ of the tokens processed.

Results at $p = 0.8$ are shown in Fig.~\ref{fig:math_arena}. Strikingly, solutions from two \emph{different} models to the \emph{same} question show only $\sim$60\% expert overlap on average, comparable to the overlap between solutions from the \emph{same} model on \emph{different} questions, and far below the same-model same-question baseline.
This gap is consistent across all layers and all three evaluated MoEs 
In terms of expert activation, solving the same problem with a different model is as ``foreign'' as solving an entirely different problem.
Results under other values of $p$ are shown in Fig.~\ref{fig:overlap_at_p_more}, where similar patterns hold.

This highlights the difficulty of making micro-level claims about expert specialization: even when two trajectories address identical mathematical content, model-specific factors can drive expert usage further apart than content differences do.
\vspace{-4pt}
\subsection{Rollout content dominates expert usage}
\vspace{-4pt}

\begin{figure}[!t]
\centering
\includegraphics[width=\linewidth]{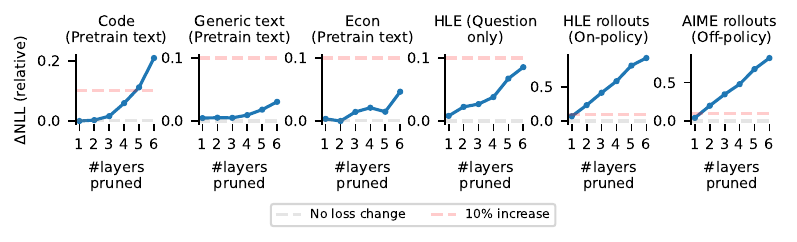}\\[-2ex]
    \caption{\small{On gpt-oss-20b, we keep the $12$ most-used experts (out of 32) in the last 
    $k$ layers, using usage frequency from a single hand-written paragraph.
    \textbf{First four columns:}
    Pruning up to 2 layers has a negligible effect 
    ($<$10\% NLL increase) on prompt-only inputs, indicating router collapse during 
    prefilling. 
    \textbf{Last two columns:} 
    Adding completions makes the pruned experts consequential, with NLL increasing substantially for on- or off-policy rollout.
    }}
    \label{fig:pruning}
% \vspace{-6pt}
\end{figure}

\begin{figure}[!t]
\centering
    \includegraphics[width=\linewidth]{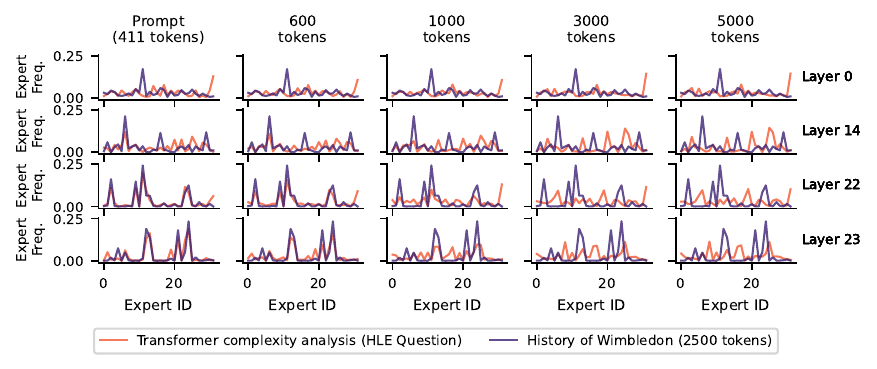}\\[-2ex]
    \caption{\small{
    Per-expert activation frequencies for an HLE question (orange, 411 tokens) and a Wimbledon Wikipedia passage (purple, 2500 tokens) across four layers and five generation lengths for answering the HLE question.
    During prefilling, the two unrelated sequences activate nearly identical experts in deeper layers. As the generation proceeds, their activation profiles diverge, with the split emerging earlier in later layers.
    % While GPT-OSS shows no specialization in later layers for just the prompt, the expert usage pattern starts to diverge between the generic corpus and the generated rollout as reasoning proceeds.
    }}
    \label{fig:hle_vs_wlmbledon}
\end{figure}

We prompt an instruction-tuned MoE \citep[\texttt{Ling-mini}][]{ling} with query pairs that are matched along one axis but differ along another: (i) different concepts explained in the same style; (ii) the same concept explained in plain v.s mathematical language; (iii) AIME problems prompted to be solved in different languages.
We then track expert-usage similarity between paired queries' evolution from the prompt tokens through progressively longer rollouts.

Fig.~\ref{fig:rollout_vs_expert_usage} reveals that during the prefill phase, all three query pairs show nearly identical expert usage. What happens during generation, however, depends on the query: for some pairs the similarity remains high throughout, while for others it drops sharply as the rollout lengthens. Critically, which regime applies is not predictable from the prompt alone. 

The implication is consequential for any analysis that inspects only the prompt: prefill-phase expert usage is not a reliable proxy for full generation usage pattern, as rollout content is unknown in advance, characterizing specialization requires observing complete trajectories.

\vspace{-4pt}
\subsection{Router collapse for prefilling stage}
\vspace{-4pt}

\begin{figure}[!t]
\centering
    \includegraphics[width=0.65\linewidth]{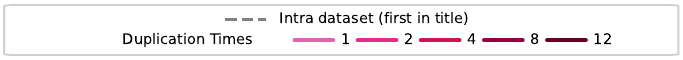}
    % \centering
    % \includegraphics[width=\linewidth]{figures/duplication/Ling-mini-base-2.0_expert_overlap_duplication_cosine_one_line.pdf}
    % \includegraphics[width=\linewidth]{figures/duplication/gpt-oss-20b_expert_overlap_duplication_cosine_one_line.pdf}
    \begin{subfigure}[t]{\textwidth}
    \centering
    \includegraphics[width=\linewidth]{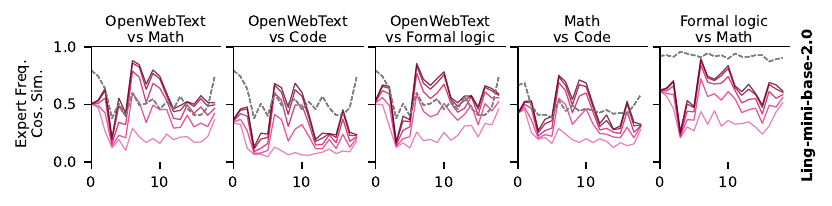}
     \caption{Ling-mini-base, no chat template}
     \label{fig:gpt_oss_duplication_no_chat_template}
    \end{subfigure}
    \begin{subfigure}[t]{\textwidth}
    \centering
     \includegraphics[width=\linewidth]{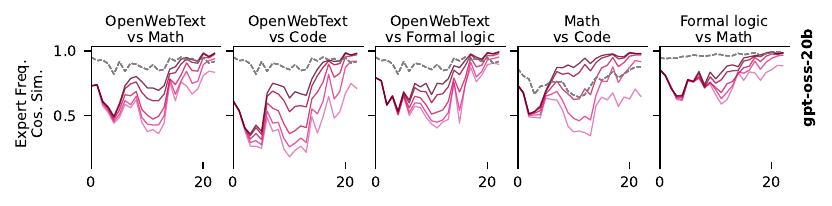}
     \caption{gpt-oss-20b, no chat template}
     \label{fig:gpt_oss_duplication_no_chat_template}
    \end{subfigure}
    \begin{subfigure}[t]{\textwidth}
        \includegraphics[width=\linewidth]{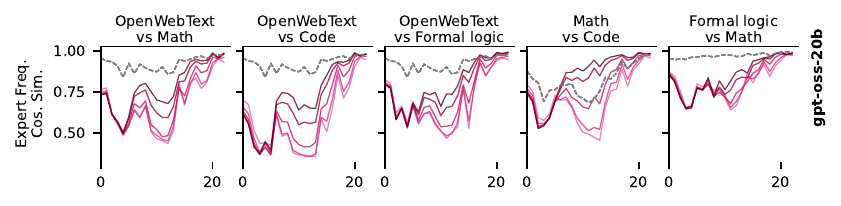}
        \caption{gpt-oss-20b, with chat template}
        \label{fig:gpt_oss_duplication_chat_template}
    \end{subfigure}
    \begin{subfigure}[t]{\textwidth}
        \includegraphics[width=\linewidth]{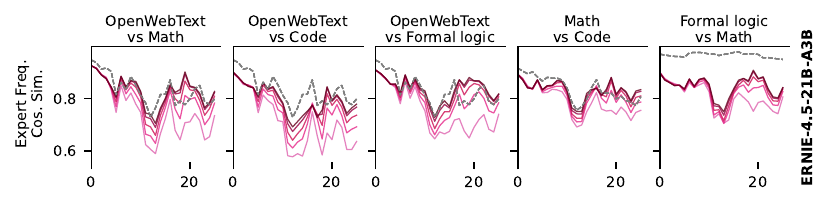}
        \caption{ERNIE-4.5-21B-A3B-Base, with chat template}
    \end{subfigure}

    \caption{\small{
    Cross-domain expert-usage similarity under increasing input duplication 
    (1$\times$ to 12$\times$) across models; dashed grey shows intra-domain similarity.
    For models except gpt-oss with chat template (\ref{fig:gpt_oss_duplication_chat_template}), hidden state similarity alone does not produce router collapse under no duplication. 
    Duplication amplifies the dominant hidden-state direction, progressively 
    collapsing cross-domain expert usage toward the intra-domain baseline. 
    % Additional models' results are shown in Fig.~\ref{fig:more_duplication_collapse}.
    % completely so in the deeper layers of \texttt{gpt-oss-20b}. 
    % Results for gpt-oss are shown in Fig.~\ref{fig:gpt_oss_duplication_with_chat_template} in the Appendix, .
    }}
    % When the model does not naturally show collapsed expert
    % Embedding collapses at prefill stage do not necessarily lead to router collapse.
    % However we can duplicate sequences to amplify the shared direction to make expert usage similar (mechanism explained in Fig.~\ref{fig:duplication_mechanism}) }
    % \xw{todo: look at the model output after duplication}}
    \label{fig:duplication_gpt_ling}
\end{figure}

Our third finding is perhaps the most counterintuitive: during prefilling, semantically unrelated sequences can activate exactly the same experts.

As an illustrative example, the bottom left of Fig.~\ref{fig:hle_vs_wlmbledon} shows that in the later layers of \texttt{gpt-oss-20b}, a question from the Humanity's Last Exam \citep[HLE,][]{phan2025humanity} and a passage on the history of Wimbledon engage an identical set of experts.
To probe how widespread this ``router collapse'' is, we adopt an extreme pruning protocol: using a single hand-written paragraph as the validation data, we retain only the top-12 experts at the bottom layers and examine NLL changes across a variety of corpora.
The first four columns of Fig.~\ref{fig:pruning} shows that pruning experts in the last two layers has a negligible effect (${<}10\%$ NLL increase) for prompt-only input across four different topics.

The picture changes once the model begins generating. Bottom right in Fig.~\ref{fig:hle_vs_wlmbledon} shows that expert-usage patterns indistinguishable during prefilling begin to diverge as the reasoning chain unfolds; and the last two columns of Fig.~\ref{fig:pruning} confirm that the pruned experts are highly consequential for the model output's (conditional) likelihoods, both on- and off-policy. The hidden state similarity across depth (Fig.~\ref{fig:gpt_oss_reasoning_question_hidden_state_sim}, appendix) further verifies the phase change.
% tells a similar story: CHidden states of prompt parts show  to $1$ in deeper layers while the rollout part shows less similarity.

We hypothesize that router collapse during prefilling is a property arising from gpt-oss's structured reasoning output: The input-invariant initial decoding tokens (examples shown in App.~\ref{app:gpt_sampled_output}) are largely controlled by the final-layer hidden states, which converge to similar representations across diverse prompts, causing the router to select the same experts. Consistent with this, we do not observe a comparable collapse in a non-reasoning instruction-tuned model (\texttt{ERNIE-4.5-21B-A3B-PT}; Fig.~\ref{fig:ernie_duplication_with_chat_template}).
Removing the chat template also reduces the collapse (lightest line, Fig.~\ref{fig:gpt_oss_duplication_no_chat_template} and leftmost column, top figure in Fig.~\ref{fig:gpt_oss_duplication_with_chat_template}), though expert-usage similarity still increases with depth in those models, likely due to the hidden state correlation discussed in the previous section.
Finally, Fig.~\ref{fig:duplication_gpt_ling} shows that this collapse is not inevitable: for models that do not exhibit it naturally, prefill-stage hidden state similarity alone is insufficient to produce router collapse. 
However, duplicating the input sequences (excluding the chat template) amplifies the shared direction in the hidden states enough to force expert-usage convergence even across semantically unrelated inputs (Fig.~\ref{fig:duplication_gpt_ling}).

Taken together, these three examples illustrate why deeper layers can behave differently from shallower ones: while early-layer hidden states reflect the token distribution and thus the input's semantics, later-layer representations may instead encode the model's anticipated output, rendering the mapping from surface content to expert usage opaque.

% Note that since this is a property induced by hidden state

\section{Conclusion}

% \paragraph{Summarization} 
Expert specialization is crucial for MoE's practical performance:
First, it confirms that the model is performing proper feature learning, producing diverse hidden states for different inputs.
Secondly, from the MoE perspective, expert specialization is a necessary condition for parameter efficiency: If the router fails to route different inputs to different experts, all experts will end up being duplicates of each other, effectively wasting parameters.
However, our observation indicates that, unless an explicit objective is introduced (e.g. \citet{li2026expert}), we should not expect the hidden state and router logit space to have a metric system aligned with human standards.

\bibliography{refs}
\bibliographystyle{colm2026_conference}

\newpage
\appendix

% \paragraph{Take away messages for pre-training} Proper depth scaling is important \citep{yang2023tensorprogramsvifeature}. In dense model, improper depth scaling causes inefficient depth usage; In MoE, it will also cause inefficient width usage.

% \paragraph{Take away messages for downstream tasks of MoE}

\section{Summarization of analysis metric}\label{app:all_metrics}

\paragraph{RMS Distance}
In Fig.~\ref{fig:rms}, we compared tokens' hidden states RMS distance
with their logits' RMS distance. Given two vectors
$\bx_i, \bx_j \in \mathbb{R}^{D}$, their RMS distance is defined as
\begin{equation}\label{eq:rms_def}
    \mathrm{RMSDist}(\bx_i, \bx_j)
    =
    \sqrt{\frac{1}{D}\sum_{d=1}^{D}\bigl(x_{i,d} - x_{j,d}\bigr)^{2}},
\end{equation}
where $x_{i,d}$ denotes the $d$-th coordinate of $\bx_i$.

% \section{Understanding GPT-OSS router}
% Denote the pre-MLP features of an $N$-token sequence by $\bX \in \mathbb{R}^{N\times D}$.
% Empirically, the rows of $\bX$ become increasingly correlated in deeper layers (Fig.~???). This can be problematic because correlated token representations may induce correlated router outputs and overlapping expert usage patterns, leading to router imbalance. However, in practice, we observe that the router resists this increase in embedding correlation, keeping its outputs and expert usage relatively diverse even in deeper layers (Fig.~???). This suggests that the router may internally suppress the dominant shared directions in $\bX$.

\paragraph{Directional energy and retained energy} In the left most and right most columns of Fig.~\ref{fig:gpt_oss_correlation_depth},\ref{fig:corr_appends_dpsk},\ref{fig:corr_appends_Moonlight},\ref{fig:corr_appends_Trinity}, we studied the directional energy $\mathrm{Energy}_{\ell:\rho}$ and the retained energy for the first directions across layers, averaged over 16 sequences.
For a $T$-token sequence, at a particular layer, let \(H \in \mathbb{R}^{T\times D}\) be a data matrix whose rows are \(\bh_1^\top,\dots,\bh_T^\top\), i.e. the hidden state of the $t$th token and let
\begin{equation}
H = \bm{U}\diag{s_1,\ldots,s_D}\bm{V}^\top,
\end{equation}
denotes its SVD. The directional energy is defined as
\begin{equation}\label{eq:directional_energy}
\mathrm{Energy}_{\ell:\rho}
=
\frac{\sum_{i=\ell}^{\rho} s_i^2}{\sum_{j=1}^{r} s_j^2},
\end{equation}
which measures how much of the total feature energy is captured by singular directions $\ell,\ldots,\rho$. 
Then, given a router $\Router$ and the $i$th right singular direction $\bm{v}_i$, we define
\begin{equation}\label{eq:retained_energy}
\mathrm{RetainedEnergy}_i
=
\frac{1}{T} \sum_{t=1}^{T}\frac{\|\Router \bm{v}_i \odot \bm{m}_t|_2^2}{\|\Router\|_F^2}.
\end{equation}
where $\bm{m}_t \in \{0,1\}^E$ is a binary vector denoting whether token $t$ usesd the $e$th vector.
Since
\begin{equation}
\|\Router \bm{v}_i \odot \bm{m}_t \|_2^2 
=
\sum_{e=1}^E \bm{m}_t^e \langle \router_e, \bm{v}_i\rangle^2,
\end{equation}
this quantity measures how strongly the router preserves direction $\bm{v}_i$. A small value of $\mathrm{RetainedEnergy}_i$ for dominant shared directions (e.g., $\bm{v}_1$) suggests that the router suppresses those directions, which may help maintain balanced expert usage.

% where the singular directions

% To study this, we project the router weights onto the leading right singular directions of $\bX$. In particular, let

% where $r = \text{rank}(\bX)$, $\bm{U} \in \mathbb{R}^{N\times r}$, $\bm{\Sigma} = \diag{s_1,\ldots,s_r} \in \mathbb{R}^{r\times r}$, and $\bm{V} = [\bm{v}_1,\ldots,\bm{v}_r] \in \mathbb{R}^{D\times r}$. Since $\bX$ is not centered before taking the SVD, $\bm{v}_1$ should be interpreted as the dominant right singular direction rather than the top PCA direction; in practice, it often aligns closely with the mean direction of $\bX$ across tokens.

% Now denote the router for $E$ experts by
% \begin{equation}
% \Router = [\router_1,\ldots,\router_E]^\top \in \mathbb{R}^{E\times D},
% \end{equation}
% where $\router_e^\top$ is the weight vector for expert $e$.

% We compare two quantities. First,
% \begin{equation}\label{eq:directional_energy}
% \mathrm{Energy}_{\ell:\rho}
% =
% \frac{\sum_{i=\ell}^{\rho} s_i^2}{\sum_{j=1}^{r} s_j^2},
% \end{equation}
% which measures how much of the total feature energy is captured by singular directions $\ell,\ldots,\rho$.

\paragraph{Average cosine similarity (inter-sequence)} Given a sequence of $T$ tokens, denote the hidden states / router logits at a particular layer as $\{\bx_t\}_{t=1}^T$ , average cosine similarity computes the average over the lower triangular of the kernel matrix, defined as:
\begin{equation}\label{eq:token_cos_def}
    \overline{\mathrm{cos}}(\{\bx_t\}) 
    = \frac{2}{T(T-1)} \sum_{i=2}^{T} \sum_{j=1}^{i-1}
      \frac{\bx_i^{\top} \bx_j}{\|\bx_i\|\,\|\bx_j\|}.
\end{equation}

\paragraph{Router Hamming Similarity}
Given a sequence of $T$ tokens, denote every token's expert usage at a
particular layer as $\{\bm{m}_t\}_{t=1}^{T}$, where
$\bm{m}_t \in \{0,1\}^{E}$ is a binary vector denoting whether token $t$
used the $e$-th expert. The router Hamming similarity is computed as
\begin{equation}\label{eq:hamming_def}
    \overline{\mathrm{Hamming}}(\{\bm{m}_t\})
    = \frac{2}{T(T-1)} \sum_{i=2}^{T} \sum_{j=1}^{i-1}
      \frac{\bm{m}_i^{\top} \bm{m}_j}{E},
\end{equation}
where we average over all pairwise Hamming similarity. Compared with Eq.~\eqref{eq:token_cos_def}, the hamming similarity measures expert usage similarity more directly.

\paragraph{Router confidence}
In Fig.~\ref{fig:ood_inputs}, we looked at the router confidence averaged
over 16 sequences. For each sequence of $T$ tokens, at a particular layer,
the router confidence is defined as the average maximum softmax probability
over all tokens:
\begin{equation}\label{eq:conf_def}
    \frac{1}{T} \sum_{t=1}^{T} \max_{e} \, \softmax(\Router \bh_t)_e,
\end{equation}
where $\Router$ denotes the router and $\bh_t$ denotes the hidden state.

% \text{RetainedEnergy}_{\ell:\rho}= \frac{1}{E}\sum_{e=1}^{E}\frac{\sum_{i=\ell}^\rho \router_e^\top \bm{v}_i}{\sum_{i=1}^r \router_e^\top \bm{v}_i}

\paragraph{Sequence-level expert frequency cosine similarity}
In Figs.~\ref{fig:pooled_embd}, \ref{fig:rollout_vs_expert_usage},
\ref{fig:duplication_gpt_ling}, we compared two different sequences'
expert usage frequency. In particular, for each sequence of $T$ tokens,
we count the frequency of token visits for each expert and normalize it
into a probability vector:
\begin{equation}
    \bp^{(s)} = \frac{1}{TK} \sum_{t=1}^{T} \bm{m}_t^{(s)}
    \in \Delta^{E-1},
\end{equation}
where $K$ denotes the number of activated experts and $\bm{m}_t^{(s)} \in \{0,1\}^E$ is the expert-usage vector of token
$t$ in sequence $s$, and $\Delta^{E-1}$ denotes the $(E-1)$-simplex.
Then we compute the cosine similarity between two frequency vectors
$\bp^{(s)}$ and $\bp^{(s')}$,
\begin{equation}\label{eq:seq_expert_freq_sim}
    \mathrm{FreqSim}(s, s')
    = \frac{\bp^{(s)\top} \bp^{(s')}}{\|\bp^{(s)}\|\,\|\bp^{(s')}\|},
\end{equation}
as a measure of sequence-level expert usage difference.

\paragraph{Pooled hidden states cosine similarity}
In Fig.~\ref{fig:pooled_embd}, the $x$-axis denotes the cosine similarity
between two sequences' pooled hidden states. Given two sequences $i$ and
$j$ of length $L_i$ and $L_j$ respectively, we first normalise each
token's hidden state at layer $\ell$:
\begin{equation}
    \tilde{\bh}_t^\ell = \frac{\bh_t^\ell}{\|\bh_t^\ell\|_2 + \varepsilon},
\end{equation}
and pool across tokens to obtain a sequence-level representation:
\begin{equation}
    \bar{\bh}_i^\ell = \frac{1}{L_i} \sum_{t=1}^{L_i} \tilde{\bh}_t^\ell.
\end{equation}
The pooled cosine similarity between sequences $i$ and $j$ is then:
\begin{equation}\label{eq:seq_pooled_sim}
    d_{i,j}^\ell = \cos\!\left(\bar{\bh}_i^\ell,\, \bar{\bh}_j^\ell\right),
\end{equation}
where $\varepsilon > 0$ is a small constant for numerical stability.

\paragraph{Overlap under top-$P$}
Given two sequences $s$ and $s'$, we look at the overlap ratio of their
top-$P$ expert sets. Using the empirical expert frequency vector
$\bp^{(s)}$ defined above, we define the top-$P$ expert set of sequence
$s$ as the smallest set of experts (sorted by decreasing frequency) whose
cumulative probability mass exceeds $P$:
\begin{equation}
    \mathcal{E}_P^{(s)} = \underset{S \subseteq [E]}{\arg\min} \; |S|
    \quad \text{s.t.} \quad
    \sum_{e \in S} p_e^{(s)} \geq P, \quad p_e^{(s)} \geq p_{e'}^{(s)}
    \; \forall e \in S, e' \notin S,
\end{equation}
and the overlap ratio between sequences $s$ and $s'$ is then their Jaccard similarity:
\begin{equation}\label{eq:overlap_at_p}
    \mathrm{Overlap}_P(s, s') = 
    \frac{|\mathcal{E}_P^{(s)} \cap \mathcal{E}_P^{(s')}|}{|\mathcal{E}_P^{(s)} \cup \mathcal{E}_P^{(s')}|}.
\end{equation}
Compared with cosine similarity, Eq.~\eqref{eq:seq_expert_freq_sim}, $\mathrm{Overlap}_P(s, s')$ measures expert usage similarity more directly, and the value is more interpretable.

% \paragraph{\Delta NLL}

\section{Expert properties arise from the embeddings}\label{app:proof}
Given two hidden states $x_i, x_j \in \R^D$ and the router $\Router \in \R^{E \times D}$.

Naively, we have
\begin{equation}
    \norm{\Router x_i - \Router x_j}_2 \leq \norm{\Router}_2 \norm{x_i - x_j}_2
\end{equation}
this bound is too loose, and it does not take into account data geometry. We can introduce an orthogonal base $Q \in \R^{D \times D}$, and we look at
\begin{equation}
     \norm{(\Router Q^\top) Q x_i -  (\Router Q^\top)Q x_j}_2 \leq \norm{\Router Q}_2 \norm{Q(x_i - x_j)}_2
\end{equation}

now if $x_i - x_j$ lies on $\text{span}(Q)$, we have $\norm{Q(x_i - x_j)}_2 = \norm{x_i - x_j}_2$.

To construct such $Q$, we could consider having $Q$ as the right singular matrix of $\bX$. This is formalized in the next lemma.

% Note that this bound could still be too loose as we are weighting each direction equally. Instead we can have 
% \begin{equation}
% Q = \sum_{i=1}^r \alpha_i \bm{v}_i\bm{v}_i^\top,~~\text{where } \alpha_i = \frac{\sigma_i}{\sum_{j=1}^r \sigma_j }
% \end{equation}

% ----- formal lemma\\

\begin{lemma}[Data-aware Lipschitz bound via principal subspace projection]\label{lemma:sharp_lipschitz}
Let \(X \in \mathbb{R}^{N\times D}\) be a data matrix whose rows are \(x_1^\top,\dots,x_N^\top\), and let
\[
X = U \Sigma V^\top
\]
be its singular value decomposition. Denote by $m=\min\{N,D\}$ the rank of $X$. Fix \(r \in \{1,2, \dots, m\}\), and let
\[
V_r \in \mathbb{R}^{D\times r}
\]
be the matrix formed by the top \(r\) right singular vectors of \(X\). Define the orthogonal projector onto the corresponding principal subspace by
\[
\Pi_r := V_r V_r^\top \in \mathbb{R}^{D\times D}.
\]
Let \(P \in \mathbb{R}^{E\times D}\) be the router weight matrix. Then, for any \(i,j \in \{1,\dots,N\}\),
\[
\|Px_i - Px_j\|_2
\le
\|P\Pi_r\|_2 \, \|\Pi_r(x_i-x_j)\|_2
+
\|P(I-\Pi_r)(x_i-x_j)\|_2.
\]
In particular,
\[
\|Px_i - Px_j\|_2
\le
\|P\Pi_r\|_2 \, \|x_i-x_j\|_2
+
\|P(I-\Pi_r)\|_2 \, \|(I-\Pi_r)(x_i-x_j)\|_2,
\]
and
\[
\|(I-\Pi_r)(x_i-x_j)\|_2^2
=
\sum_{k=r+1}^{D} \langle x_i-x_j, v_k\rangle^2.
\]
\end{lemma}

\begin{proof}
Fix \(i,j\), and define
\[
\Delta_{ij} := x_i - x_j \in \mathbb{R}^D.
\]
Since \(\Pi_r\) is the orthogonal projector onto \(\operatorname{span}(V_r)\), we have the decomposition
\[
\Delta_{ij} = \Pi_r \Delta_{ij} + (I-\Pi_r)\Delta_{ij}.
\]
Applying \(P\) to both sides yields
\[
P\Delta_{ij} = P\Pi_r \Delta_{ij} + P(I-\Pi_r)\Delta_{ij}.
\]
By the triangle inequality,
\[
\|P\Delta_{ij}\|_2
\le
\|P\Pi_r \Delta_{ij}\|_2 + \|P(I-\Pi_r)\Delta_{ij}\|_2.
\]
Using the operator norm bound on the first term,
\[
\|P\Pi_r \Delta_{ij}\|_2
\le
\|P\Pi_r\|_2 \, \|\Delta_{ij}\|_2,
\]
we obtain
\[
\|P\Delta_{ij}\|_2
\le
\|P\Pi_r\|_2 \, \|\Delta_{ij}\|_2
+
\|P(I-\Pi_r)\Delta_{ij}\|_2.
\]
Recalling that \(\Delta_{ij}=x_i-x_j\), this proves
\[
\|Px_i - Px_j\|_2
\le
\|P\Pi_r\|_2 \, \|x_i-x_j\|_2
+
\|P(I-\Pi_r)(x_i-x_j)\|_2.
\]

For the second bound, apply again the operator norm inequality:
\[
\|P(I-\Pi_r)(x_i-x_j)\|_2
\le
\|P(I-\Pi_r)\|_2\,\|(I-\Pi_r)(x_i-x_j)\|_2,
\]
where we have used the fact that $I - \Pi_r = (I - \Pi_r)^2$.\\

Now let \(v_1,\dots,v_D\) denote the columns of \(V\). Since \(\Pi_r = \sum_{k=1}^r v_k v_k^\top\), we have
\[
I-\Pi_r = \sum_{k=r+1}^D v_k v_k^\top.
\]
Therefore,
\[
(I-\Pi_r)(x_i-x_j)
=
\sum_{k=r+1}^D \langle x_i-x_j, v_k\rangle v_k.
\]
Because \(v_{r+1},\dots,v_D\) are orthonormal, it follows from Pythagoras' theorem that
\[
\|(I-\Pi_r)(x_i-x_j)\|_2^2
=
\sum_{k=r+1}^{D} \langle x_i-x_j, v_k\rangle^2.
\]

This concludes the proof.
\end{proof}

The result of \cref{lemma:sharp_lipschitz} suggests that when data variance is mostly concentrated on few principal directions, the lipschitz constant $\|P \Pi_r\|_2$ is much sharper than the naive constant $\|P\|_2$. Intuitively, this is the case when the hidden states exhibit significant correlation.

\subsection{Shared directions become routing-insignificant under load balancing.}\label{app:balance_loss_analysis}

Consider a router matrix \(P \in \R^{E\times D}\) with logits
\[
g(x)=Px \in \R^E,
\qquad
p(x)=\softmax(Px)\in \R^E.
\]
Since \(\softmax(z+\alpha \one)=\softmax(z)\) for every \(z\in\R^E\), \(\alpha\in\R\), only the component of \(Px\) orthogonal to \(\one\) matters for routing. Let
\[
\PiE := I_E - \frac1E \one \one^\top
\]
denote the orthogonal projector onto \(\one^\perp\). Then the \emph{effective} routing signal is \(\PiE Px\).

We study the simple correlated data model
\[
x_i = \mu + \xi_i,\qquad i=1,\dots,N,
\]
where \(\mu\in\R^D\) is a shared direction and \(\frac1N\sum_{i=1}^N \xi_i \approx 0\) (uncorrelated centered $\xi$'s). For tractability, we consider a simplified load-balancing loss 
\[
\mathcal L_{\mathrm{bal}}(P)
:=
\left\|
\frac1N \sum_{i=1}^N \softmax(Px_i) - \frac1E \one
\right\|^2.
\]
The following lemma shows that, in the small-logit regime, minimizing load balancing forces the shared direction \(\mu\) to become routing-insignificant.

\begin{lemma}\label{lemma:lbl_analyze}
Assume that \(\max_{1\le i\le N}\norm{P(\mu+\xi_i)} \le \varepsilon\) for some \(\varepsilon>0\), and that \(\frac1N\sum_{i=1}^N \xi_i \approx 0\). Then
\[
\frac1N \sum_{i=1}^N \softmax(Px_i)
\approx
\frac1E \one + \frac1E \PiE P\mu + O(\varepsilon^2),
\]
and therefore
\[
\mathcal L_{\mathrm{bal}}(P)
\approx
\frac{1}{E^2}\norm{\PiE P\mu}^2 + O(\varepsilon^3).
\]
In particular, any near-minimizer of \(\mathcal L_{\mathrm{bal}}\) must satisfy
\[
\PiE P\mu \approx 0
\qquad\Longleftrightarrow\qquad
P\mu \approx \overline{P\mu}\,\one,
\]
where \(\overline{P\mu}:=\frac1E \one^\top P\mu\).
\end{lemma}

\begin{proof}
Using the Taylor expansion of \(\softmax\) at \(0\), for any \(z\in\R^E\) with \(\norm{z}\le \varepsilon\),
\[
\softmax(z)
=
\frac1E \one + \frac1E \PiE z + O(\norm{z}^2),
\]
where the \(O(\norm{z}^2)\) term is uniform for \(\norm{z}\le \varepsilon\). Applying this with \(z_i=P(\mu+\xi_i)\), we get
\[
\softmax(Px_i)
=
\frac1E \one + \frac1E \PiE P(\mu+\xi_i) + O(\varepsilon^2).
\]
Averaging over \(i\) and using \(\frac1N\sum_i \xi_i=0\),
\[
\frac1N \sum_{i=1}^N \softmax(Px_i)
=
\frac1E \one + \frac1E \PiE P\mu + O(\varepsilon^2).
\]
Subtracting \(\frac1E\one\) and squaring yields
\[
\mathcal L_{\mathrm{bal}}(P)
=
\left\|
\frac1E \PiE P\mu + O(\varepsilon^2)
\right\|^2
=
\frac1{E^2}\norm{\PiE P\mu}^2 + O(\varepsilon^3),
\]
since \(\norm{\PiE P\mu}=O(\varepsilon)\) in the small-logit regime. This proves the claim.
\end{proof}

\paragraph{Geometric interpretation.}
The condition $P\mu = \overline{P\mu}\,\one$
means that the shared direction \(\mu\) contributes the \emph{same} logit shift to every expert. Hence \(\mu\) induces no relative preference between experts and is invisible to routing. Equivalently,
\[
\PiE P\mu = 0,
\]
so \(\mu\) lies in the kernel of the discriminative part of the router \(\PiE P\). Writing the rows of \(P\) as \(p_1^\top,\dots,p_E^\top\), this is also equivalent to
\[
\langle p_e-p_{e'},\mu\rangle = 0
\qquad \forall e,e'\in\{1,\dots,E\},
\]
showing that all expert-separating directions are orthogonal to \(\mu\). Thus the routing decision is invariant along the shared direction.

This gives a simple theoretical explanation for why highly correlated inputs tend to make the router suppress shared directions. In the model \(x_i=\mu+\xi_i\), when \(\mu\) is large, every token carries the same strong common component. Load balancing discourages such common structure from affecting routing, because it would create the same expert bias across the batch. The optimal response is therefore to make \(\mu\) \emph{routing-insignificant}, namely \(P\mu \in \mathrm{span}(\one)\). Since adding a multiple of \(\one\) does not change softmax, the router effectively ignores the shared direction and routes using only the residual variations \(\xi_i\). In more geometric terms, balanced routing factors through the quotient that removes the common correlated direction.

This analysis considered a router $P$ trained to achieve load-balancing only. Realistically, the training of $P$ is conducted simultaneously with other weight matrices in pretraining, and the load-balancing loss is an auxiliary loss. Therefore, the conclusions from this analysis should only be partially observed in practice. Specifically, instead of the router $P$ making high correlation directions fully insignificant, it should be expected that $P$ \emph{downscales} the magnitude of those directions instead.

\section{Example model outputs}

\subsection{GPT-OSS-20b}\label{app:gpt_sampled_output}
\begin{lstlisting}
<Q>Model input:</Q>
Summarize the following paragraph: [history of wimbledon]

<Q>Model output:</Q>
<|channel|>analysis<|message|>We need to summarize the text. Very long history. 
\end{lstlisting}

\begin{lstlisting}
<Q>Model input:</Q>
Think step by step, answer the question: [HLE question]

<Q>Model output:</Q>
<|channel|>analysis<|message|>We need to interpret the problem:
\end{lstlisting}

\subsection{Duplication}

Under duplications, we considered an instruction tuned model, while we did not see outputs of different inputs collapse to the same output (despite expert usage similarity increases), we did observe the output to lose randomness under random decoding.

\begin{lstlisting}
<Q>Duplication x 1:</Q>
<K>Query 1 output:</K>

If you're a coder: What language are you coding in, and what does your current stack look like?

Some key shifts:
- More attention on internal tooling (infrastructure, DevOps, internal APIs)
- Higher value on high-quality, maintainable code
- Smaller, faster iterations with
******************************
<K>Query 2 output:</K>
```
#@title Sampling
def p_sample_loop_score(model, shape, clip_denoised=True, denoised_fn=None, transform=None):
    '''
    A generic function to follow the denoising refinemnet API.
    '''
    assert transform is not None
    b =
------------------------------------------------------------
<Q>Duplication x 4:</Q>
<K>Query 1 output:</K>

OpenAI and xAI are finally deprioritizing side quests and focusing on coding, while Anthropic has been focused on coding from the start.

We're likely heading toward 3-4 top-tier coding models (including Gemini), it's not a winner-take-all market. Each model will excel
******************************
<K>Query 2 output:</K>

1. The Forward Process: The SDEThe forward diffusion process $\{ \mathbf{x}_t \}_{t=0}^T$ transforms a data sample $\mathbf{x}_0 \sim p_{data}(\mathbf{x})$ into Gaussian noise $\mathbf{x}_T \sim \mathcal{N}(\mathbf{0},
------------------------------------------------------------
<Q>Duplication x 16:</Q>
<K>Query 1 output:</K>

OpenAI and xAI are finally deprioritizing side quests and focusing on coding, while Anthropic has been focused on coding from the start.

We're likely heading toward 3-4 top-tier coding models (including Gemini), it's not a winner-take-all market. Each model will excel
******************************
<K>Query 2 output:</K>

1. The Forward Process: The SDEThe forward diffusion process $\{ \mathbf{x}_t \}_{t=0}^T$ transforms a data sample $\mathbf{x}_0 \sim p_{data}(\mathbf{x})$ into Gaussian noise $\mathbf{x}_T \sim \mathcal{N}(\mathbf{0},
------------------------------------------------------------
<Q>Duplication x 32:</Q>
<K>Query 1 output:</K>

OpenAI and xAI are finally deprioritizing side quests and focusing on coding, while Anthropic has been focused on coding from the start.

We're likely heading toward 3-4 top-tier coding models (including Gemini), it's not a winner-take-all market. Each model will excel
******************************
<K>Query 2 output:</K>

1. The Forward Process: The SDEThe forward diffusion process $\{ \mathbf{x}_t \}_{t=0}^T$ transforms a data sample $\mathbf{x}_0 \sim p_{data}(\mathbf{x})$ into Gaussian noise $\mathbf{x}_T \sim \mathcal{N}(\mathbf{0},
------------------------------------------------------------
\end{lstlisting}

\newpage

\begin{figure}
    \centering
    \includegraphics[width=.495\linewidth]{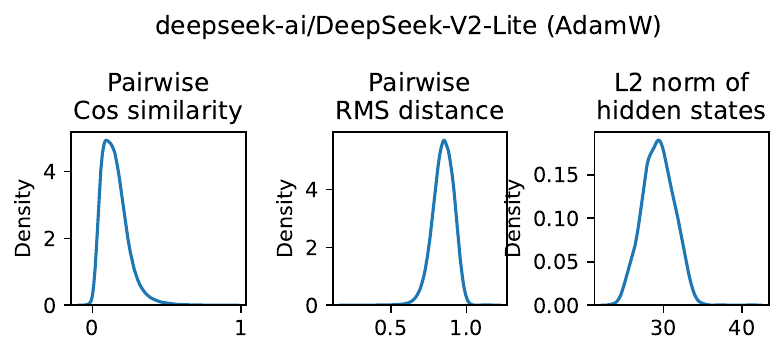}
    \includegraphics[width=.495\linewidth]{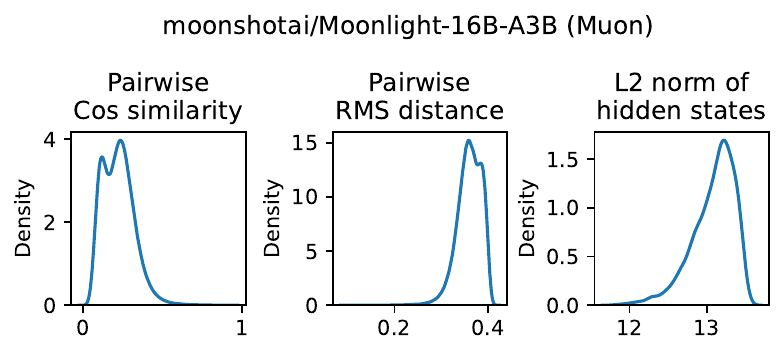}
    \caption{\textbf{Same model architecture, different optimizers induce different feature space properties.} At the same layer, given the same sequence, both models show a similar distribution of pairwise cosine similarity (over all tokens). However, embeddings in the model trained with AdamW show, on average, a higher RMS distance, indicating that different tokens exhibit high variation in norm (25 - 35). In contrast, Muon trained model shows more concentrated distribution of feature norm (12 - 14) and smaller pairwise RMS distances.}
    \label{fig:muon_vs_adamw}
\end{figure}

% \section{Token similarity increases across depth, router combined with auxiliary-loss cancels out shared direction}

\begin{figure}
    \centering
    \includegraphics[width=\linewidth]{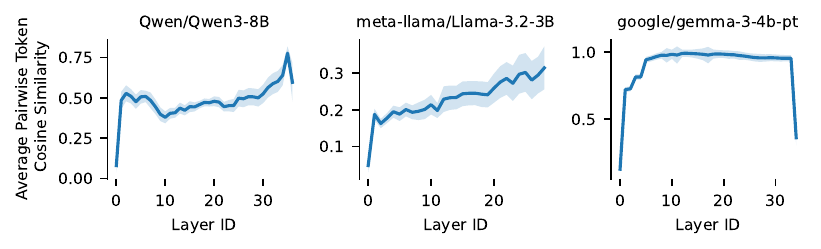}
    \caption{On dense model, except for Gemma3, which uses sandwich norm, token similarities also increase with depth.
}
    \label{fig:token_sim_vs_depth_dense}
\end{figure}

\begin{figure}
    \centering
\includegraphics[width=0.45\linewidth]{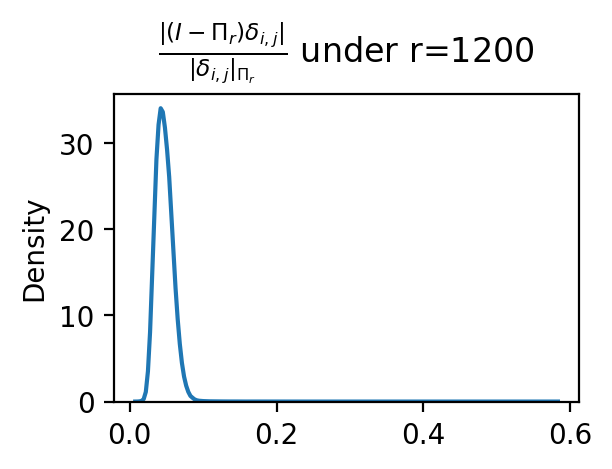}
\includegraphics[width=0.45\linewidth]{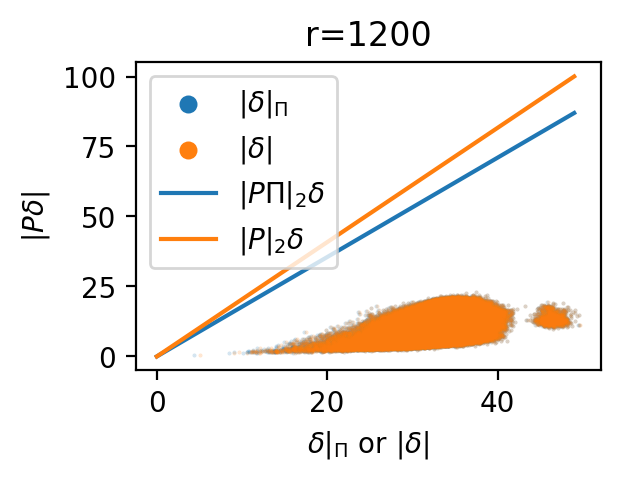}
\caption{
On Deepseek v2 lite, empirical comparison of the naive bound 
$\|P\|_2 \|\delta_{i,j}\|$ versus the tighter projection-based bound 
$\|P\Pi_r\|_2 \|\delta_{i,j}\|_{\Pi_r}$ from 
Proposition~\ref{prop:embed_logit_bound}, 
where $\delta_{i,j} = \boldsymbol{h}_i - \boldsymbol{h}_j$.
\textbf{Left:} the residual ratio 
$|(I - \Pi_r)\delta_{i,j}| \,/\, |\delta_{i,j}|_{\Pi_r}$ 
is sharply concentrated near zero under $r = 1200$, confirming that the 
residual term $\|P(I - \Pi_r)\delta_{i,j}\|_2$ is negligible and the 
bound in Eq.~\eqref{eq:norm_bound} is approximately tight with respect to the SVD 
truncation.
\textbf{Right:} the actual router logit differences $|P\delta|$ (orange 
dots) fall far below both upper bounds. While the projection-based bound 
$\|P\Pi_r\|_2\,|\delta|_{\Pi_r}$ (blue line) is theoretically tighter 
than the naive bound $\|P\|_2\,|\delta|$ (orange line), $P$ is already well-aligned with the principal subspace of the data, the gap between the two bounds is less significant in pre-trained models.
}
\label{fig:bound_viz}
\end{figure}

\begin{figure}
    \centering
\begin{subfigure}[t]{0.45\textwidth}
        \includegraphics[width=\linewidth]{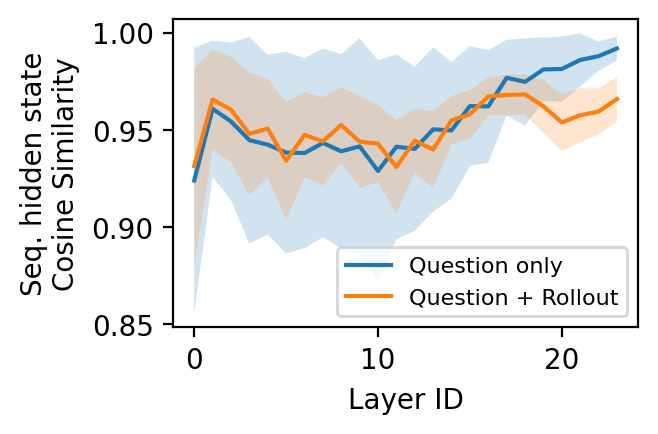}
\caption{gpt-oss-20b}
 \label{fig:gpt_oss_reasoning_question_hidden_state_sim}
\end{subfigure}
\begin{subfigure}[t]{0.45\textwidth}
    \includegraphics[width=\linewidth]{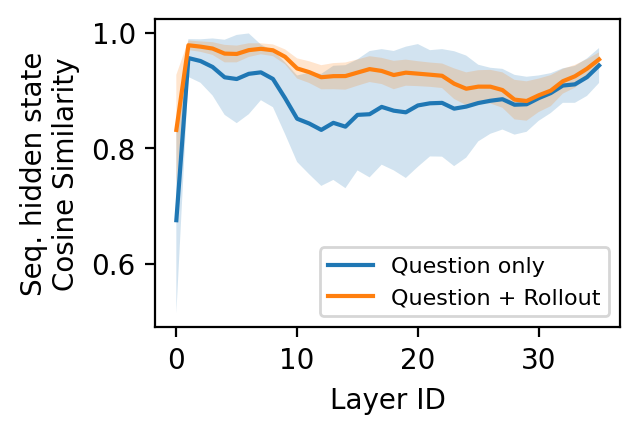}
\caption{Qwen3-4B-Thinking-2507}
\end{subfigure}
    \caption{Hidden states similarity (measured by Eq.~\eqref{eq:seq_pooled_sim}) v.s. depth for the prompt / question part v.s. question + rollout on two different reasoning models tested on 10 questions from HLE. \textbf{Left} On gpt-oss, the similarity of the rollout part falls below the question-only line and does not show collapsing, aligning with the observations in Fig.~\ref{fig:hle_vs_wlmbledon}.
    \textbf{Right} On a dense reasoning model from different lab, while we did see an increment of hidden state similarity across depth, we did not observe a clear separation in deeper layers.
    }
        \label{fig:placeholder} 
\end{figure}

\begin{figure}
    \centering
    \begin{subfigure}[t]{\textwidth}
    \centering
     \includegraphics[width=\linewidth]{figures/correlation/gpt_oss_router_inspect.pdf}
     \caption{GPT-OSS-20B (post-trained): \textbf{AdamW + explicit auxiliary loss}}\label{fig:corr_appends_gpt}
    \end{subfigure}
    \begin{subfigure}[t]{\textwidth}
    \centering
     \includegraphics[width=\linewidth]{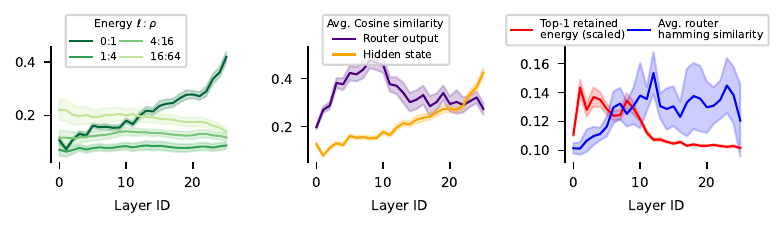}
     \caption{DeepSeek-V2-Lite (base): \textbf{AdamW + explicit auxiliary loss}}\label{fig:corr_appends_dpsk}
    \end{subfigure}
    \begin{subfigure}[t]{\textwidth}
    \centering
     \includegraphics[width=\linewidth]{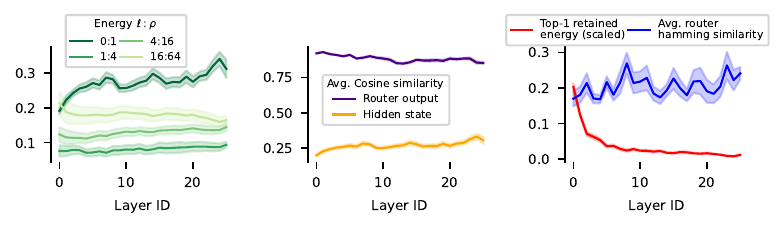}
     \caption{Moonlight 16B (structure identical to DeepSeek-V2-Lite): \textbf{Muon + auxiliary-loss-free}}\label{fig:corr_appends_Moonlight}
    \end{subfigure}
    \begin{subfigure}[t]{\textwidth}
    \centering
     \includegraphics[width=\linewidth]{figures/correlation/arcee_router_inspect.pdf}
     \caption{Trinity-Mini-Base:  \textbf{Muon + auxiliary-loss-free}}\label{fig:corr_appends_Trinity}
    \end{subfigure}
    \label{fig:muon_trained}
    \caption{Similar to Fig.~\ref{fig:gpt_oss_correlation_depth} in the main text, but for more models}
\end{figure}

\begin{figure}
\centering
\includegraphics[width=\linewidth]{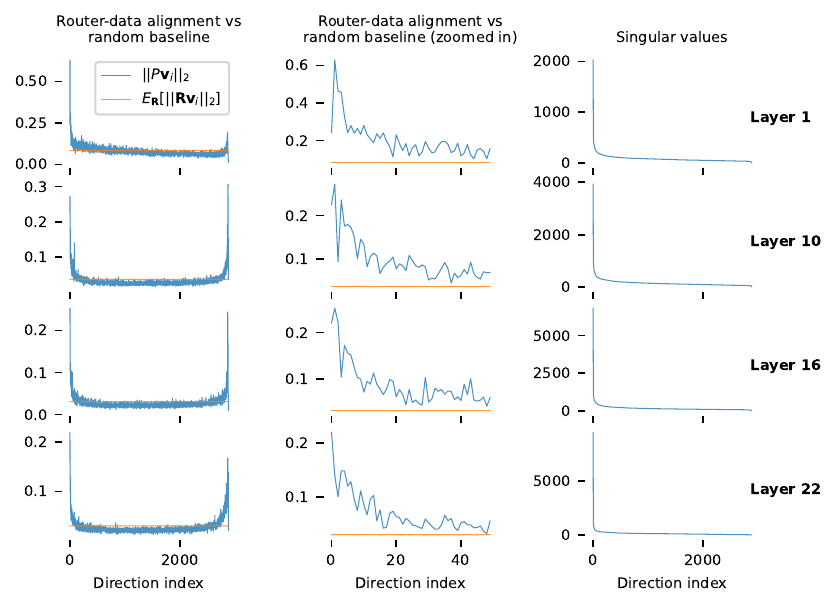}
\caption{gpt-oss-20b (w/ auxiliary loss)}
\label{fig:alignment:gpt_oss_router_data_alignment}
\end{figure}

\begin{figure}
\centering
\includegraphics[width=\linewidth]{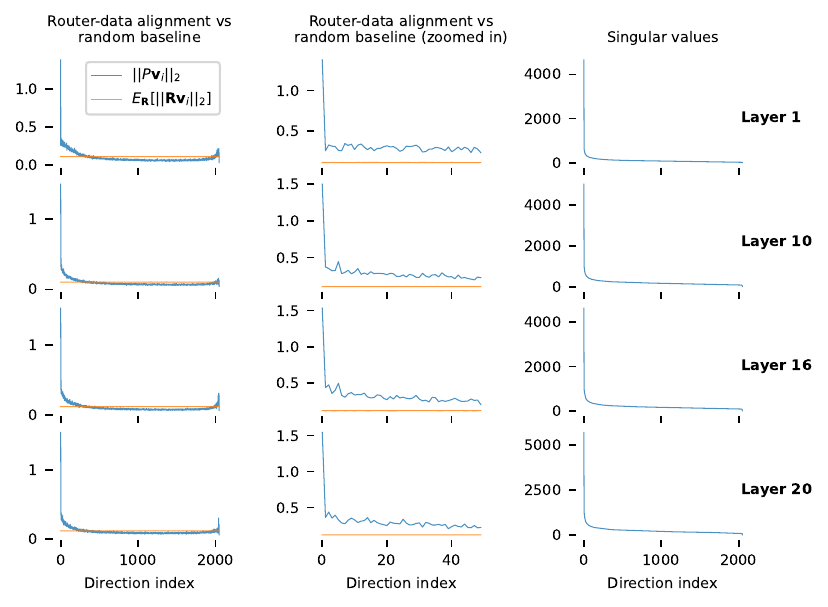}
\caption{Trinity Mini Base (Muon, no auxiliary loss)}
\label{fig:alignment:trinity_router_data_alignment}
\end{figure}

\begin{figure}
\centering
\includegraphics[width=\linewidth]{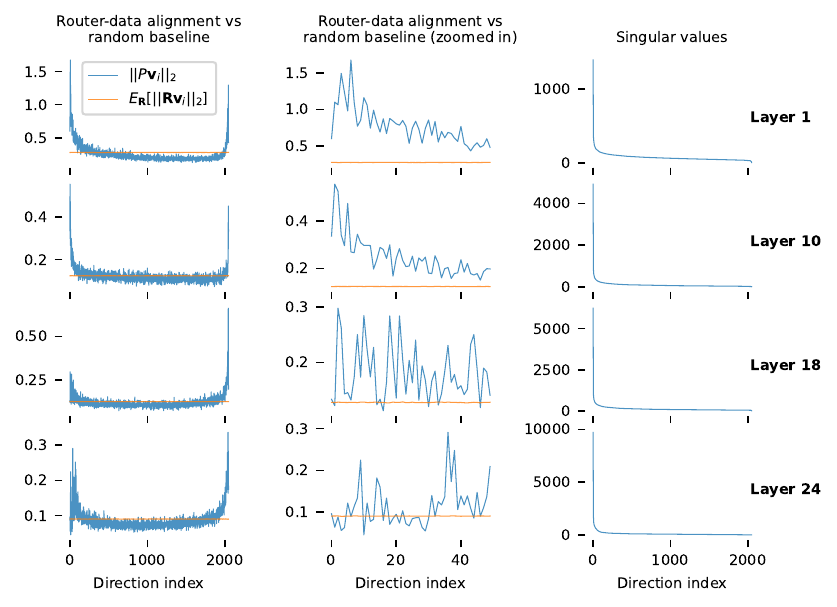}
\caption{DeepSeek V2 Lite (AdamW, w/ auxiliary loss)}
\label{fig:alignment:deepseek_router_data_alignment}
\end{figure}

\begin{figure}
\centering
\includegraphics[width=\linewidth]{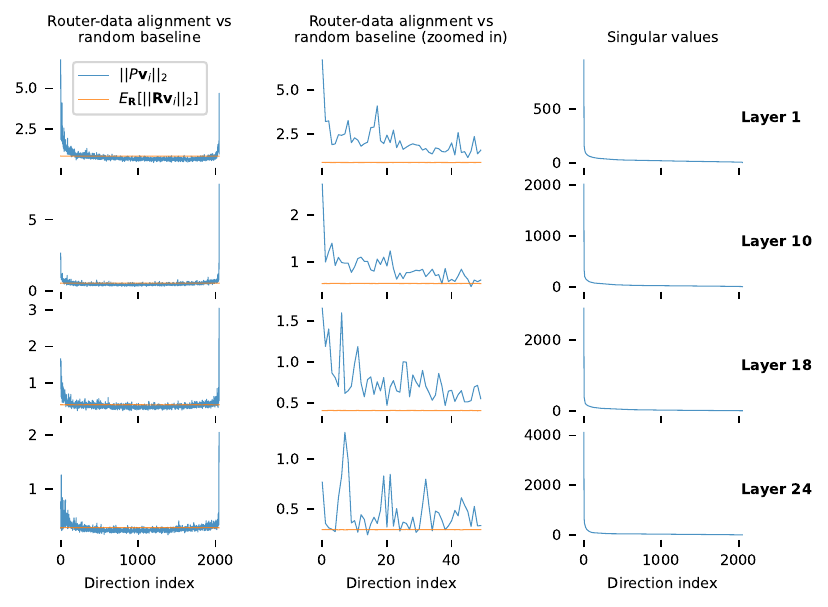}
\caption{Moonlight 16B A3B, same architecture as DeepSeek V2 Lite, trianed with Muon and no auxiliary loss}
\label{fig:alignment:moonlight_router_data_alignment}
\end{figure}

\begin{figure}
    \centering
    \includegraphics[width=\linewidth]{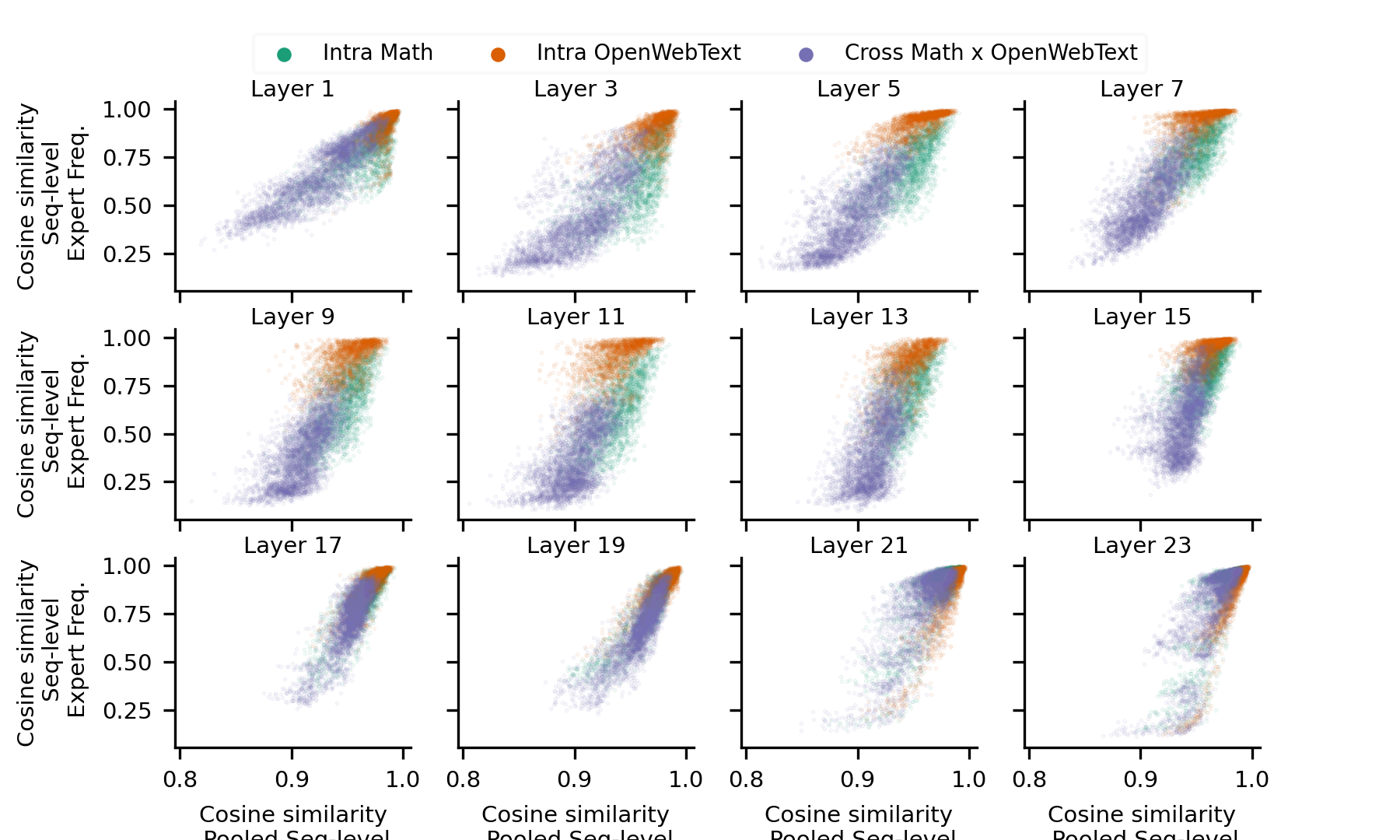}
    \caption{gpt-oss-20b}
    \label{fig:seq_emb_vs_expert_gpt}
\end{figure}

\begin{figure}
    \centering
    \includegraphics[width=\linewidth]{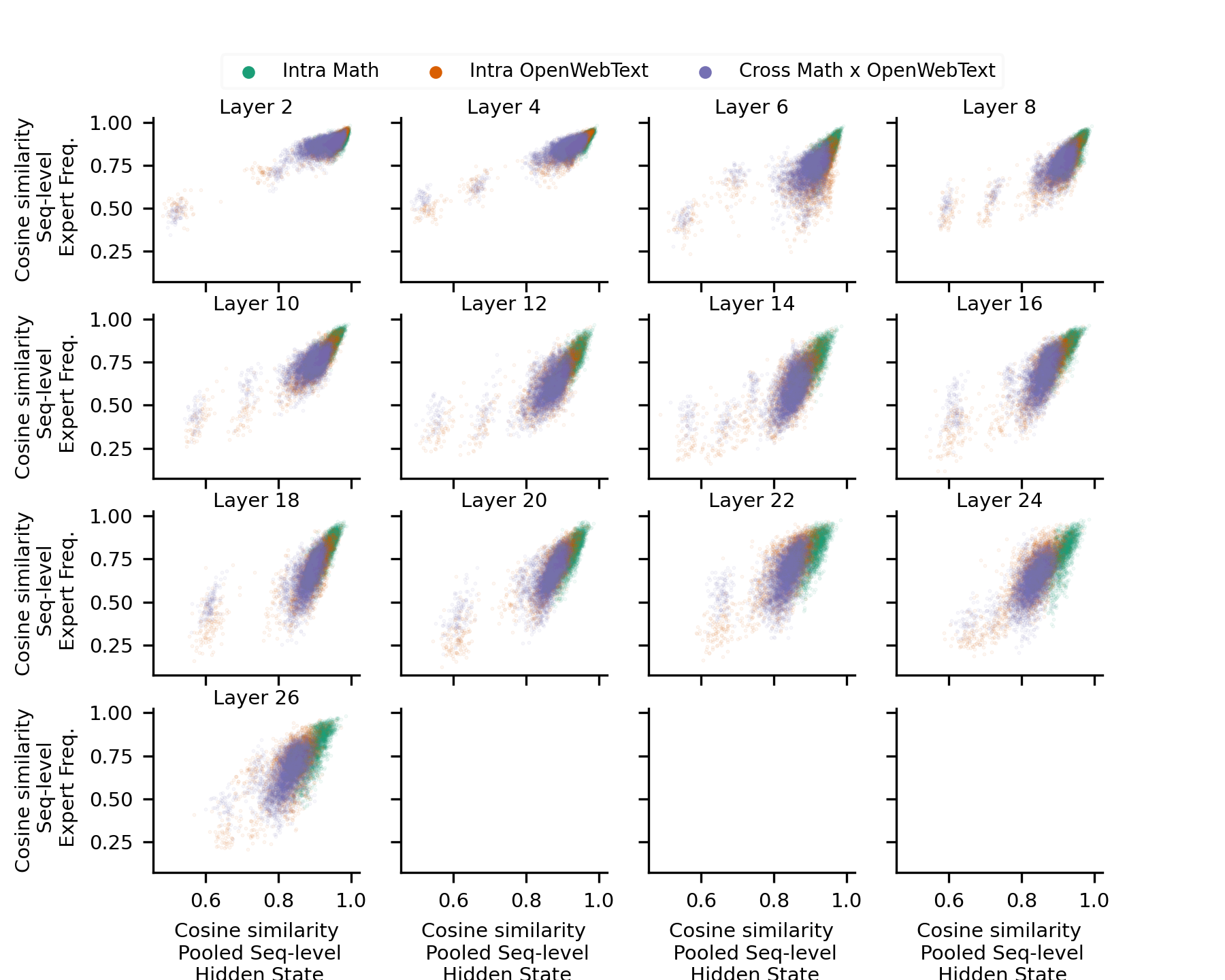}
    \caption{Ernie 4.5 21B-A3B-Base}
    \label{fig:seq_emb_vs_expert_ernie}
\end{figure}

\begin{figure}
    \centering
    \includegraphics[width=\linewidth]{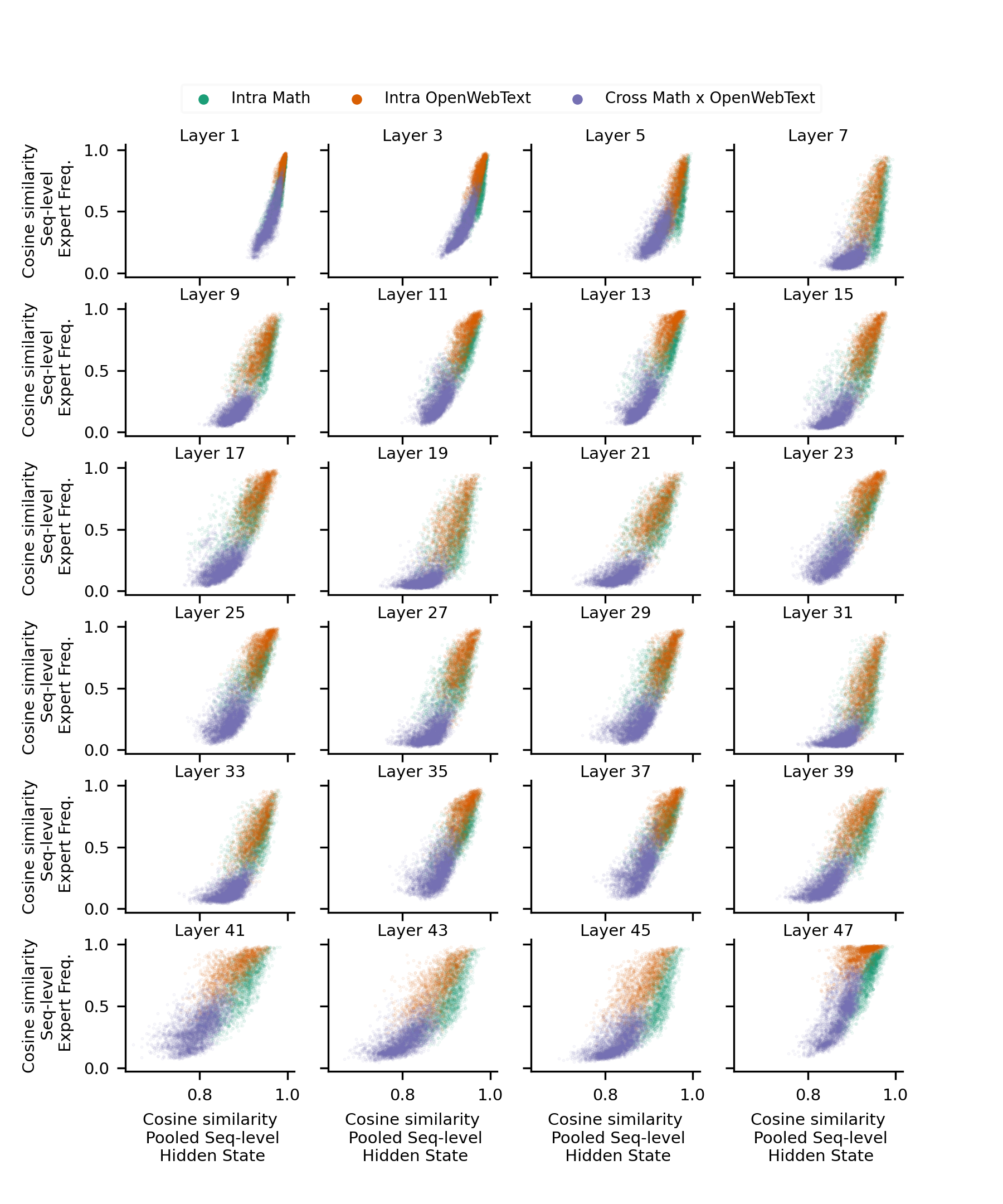}
    \caption{Qwen3-30B-A3B}
    \label{fig:seq_emb_vs_expert_qwen}
\end{figure}

\begin{figure}
    \centering
    \includegraphics[width=.45\linewidth]{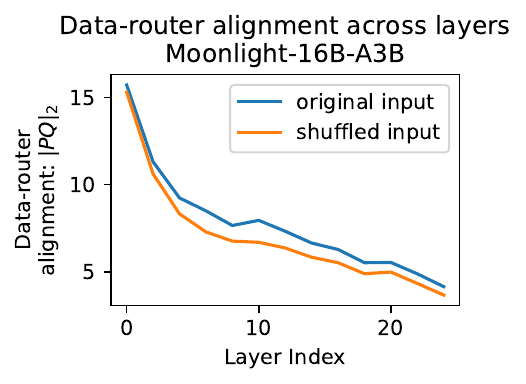}
    \includegraphics[width=.45\linewidth]{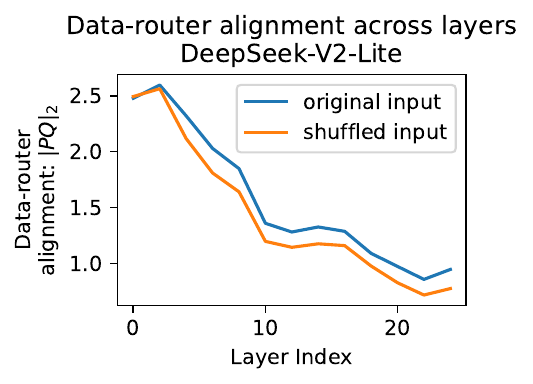}
    \caption{OOD inputs reduce router data alignment ($\norm{\Router Q}_2$), where $Q = V_r V_r^\top$ denotes the orthogonal projector onto the subspace spanned by the data's top $r$ right singular vectors.}
    \label{fig:ood_mech}
\end{figure}

\begin{figure}
    \centering
    \includegraphics[width=.49\linewidth]{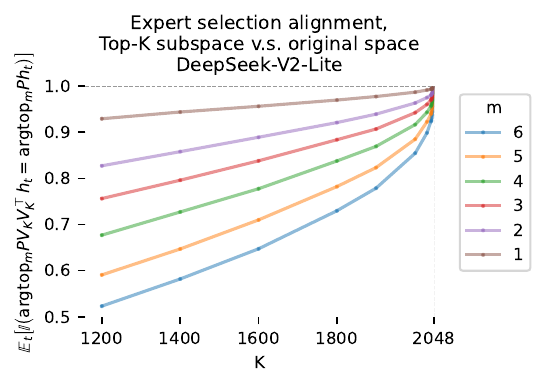}
    \includegraphics[width=.49\linewidth]{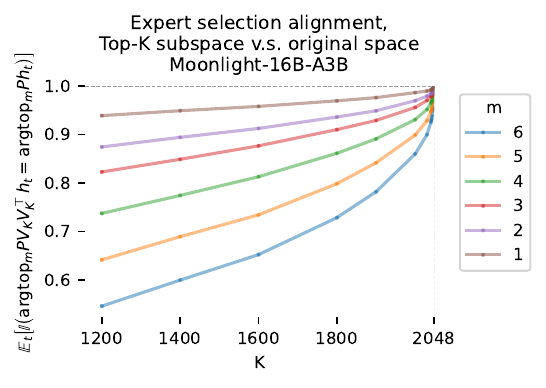}
    \caption{Keeping only top-K directions (${V}_K$) of the data (i.e. the post attention hidden states, denoted by $h_t$), we study the ratio of expert usage agreement compared with the original inputs, averaged over 50,000 tokens from \texttt{nvidia/Nemotron-CC-Math-v1} at the $10$th layer of each model.
    The expert receiving the largest logits (brown line, $m=1$) shows high agreement over all tokens.
    Lower rank experts (e.g. the 6th expert, $m=6$) demonstrate a higher degree of perturbations when the tail directions are removed; more than 40\% of the tokens end up using a different expert compared with the original inputs at $K=1200$ (out of $2048$ directions in total).
    }
    \label{fig:top_k_subspace_router_alignment}
\end{figure}

\newpage
\begin{figure}
    \centering
    \includegraphics[width=\linewidth]{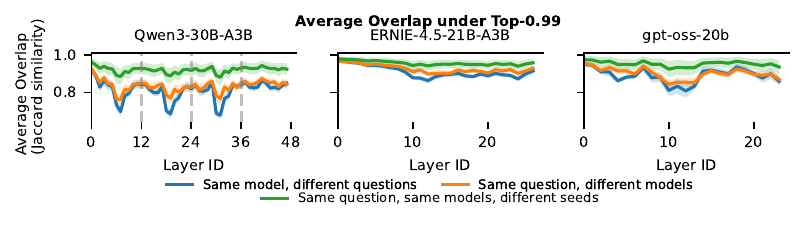}
    \includegraphics[width=\linewidth]{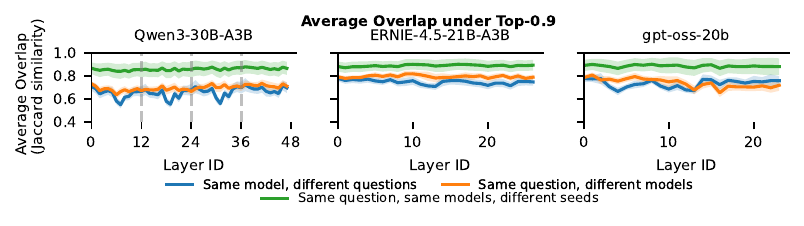}
    \includegraphics[width=\linewidth]{figures/math_arena/policy_ablation_0.8.pdf}
    \includegraphics[width=\linewidth]{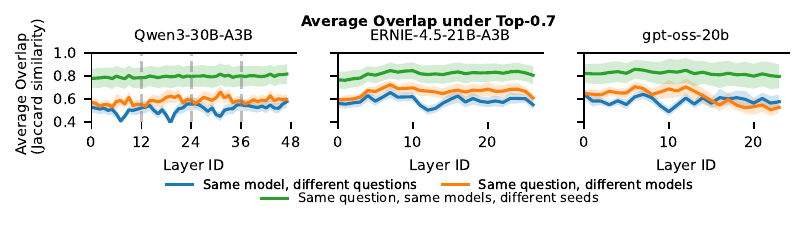}
    \includegraphics[width=\linewidth]{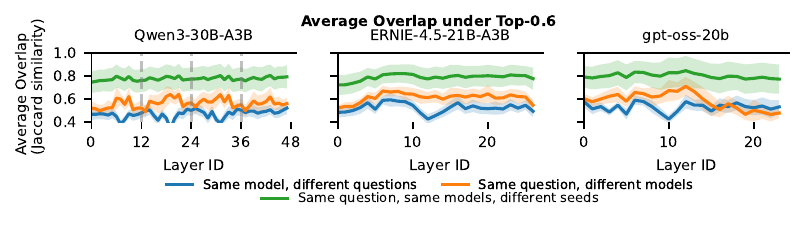}
    \caption{Similiar to Fig.~\ref{fig:math_arena} in the main text, but for more values of $p$.}
    \label{fig:overlap_at_p_more}
\end{figure}

\newpage
\begin{figure}
    \centering
    \includegraphics[width=\linewidth]{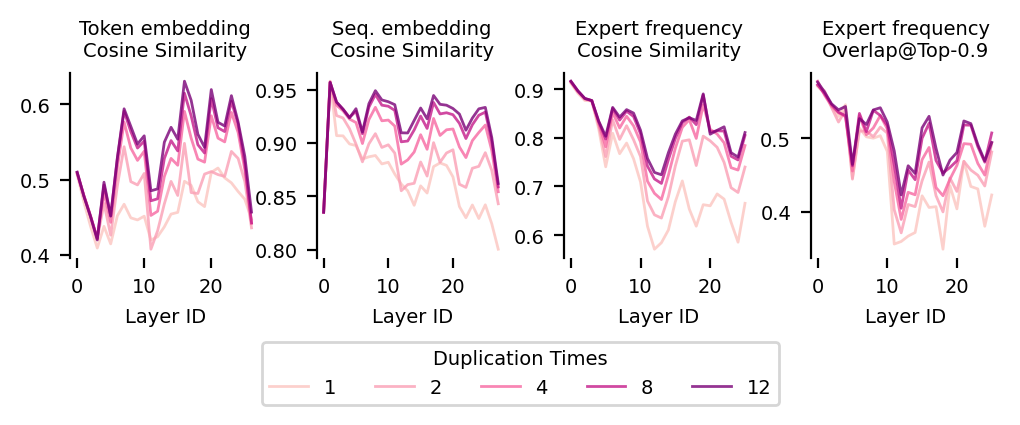}
    \caption{On Ernie-4.5, duplicating sequences increases token hidden states similarity (first column), i.e. stronger shared direction, resulting in higher sequence hidden state similarity (second column). The increment in hidden state space results in increased expert usage similarity (last two columns).}
    \label{fig:duplication_mechanism}
\end{figure}

% \begin{figure}
% \centering
%     \includegraphics[width=0.65\linewidth]{figures/duplication/legend_duplication_times.pdf}
%     \centering
    
%     \begin{subfigure}[t]{\textwidth}
%     \centering
%      \includegraphics[width=\linewidth]{figures/duplication/gpt-oss-20b_expert_overlap_duplication_cosine_one_line.pdf}
%      \caption{gpt-oss-20b, no chat template}
%      \label{fig:gpt_oss_duplication_no_chat_template}
%     \end{subfigure}
%     \begin{subfigure}[t]{\textwidth}
%         \includegraphics[width=\linewidth]{figures/duplication/gpt-oss-20b_expert_overlap_duplication_cosine_with_chat_template_one_line.pdf}
%         \caption{gpt-oss-20b, with chat template}
%         \label{fig:gpt_oss_duplication_chat_template}
%     \end{subfigure}
%     \begin{subfigure}[t]{\textwidth}
%         \includegraphics[width=\linewidth]{figures/duplication/ERNIE-4.5-21B-A3B_expert_overlap_duplication_cosine_with_chat_template_one_line.pdf}
%         \caption{ERNIE-4.5-21B-A3B-Base, with chat template}
%     \end{subfigure}
%     \caption{Similar to Fig.~\ref{fig:duplication_gpt_ling}, but we consider more models. When no duplication is introduced, gpt-oss-20b shows more collapse with chat template (b) compared with no chat template (a).}
%     \label{fig:more_duplication_collapse}
% \end{figure}

\begin{figure}
    \centering
    \includegraphics[width=\linewidth]{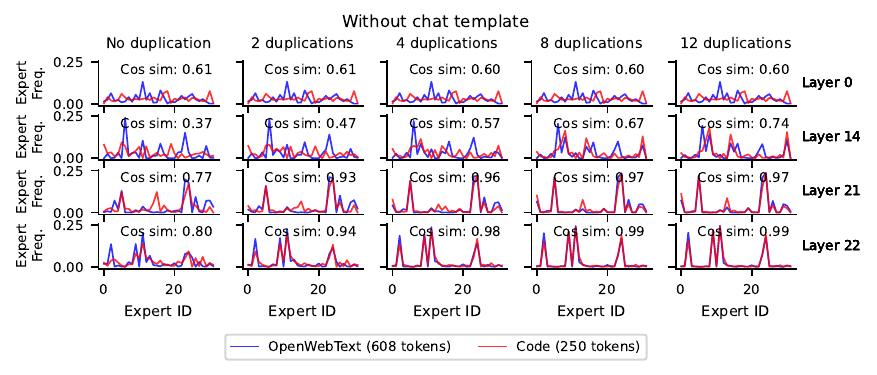}
    \includegraphics[width=\linewidth]{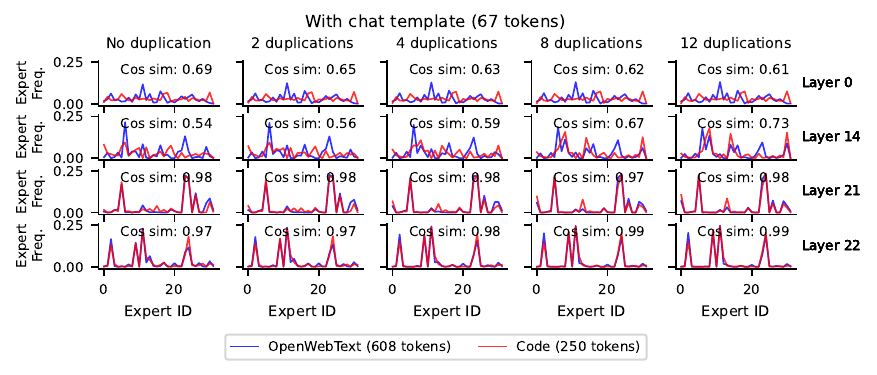}
    \caption{Comparison of expert usage under duplication for gpt-oss-20b with (bottom) and without (top) chat template on two pieces of input.}
    \label{fig:gpt_oss_duplication_with_chat_template}
\end{figure}

\begin{figure}
    \centering
    \includegraphics[width=\linewidth]{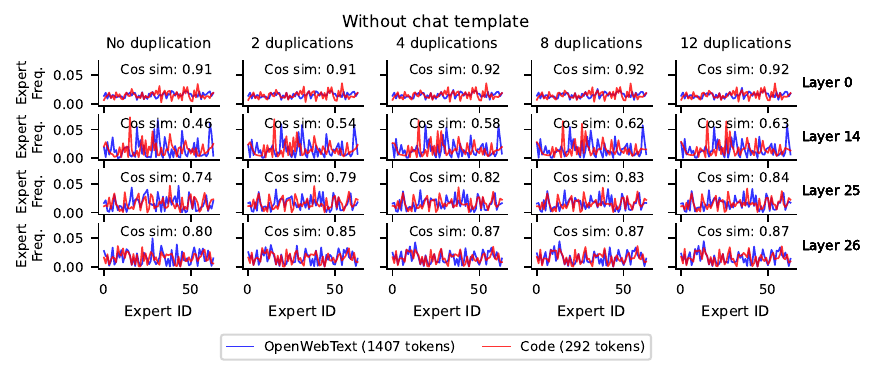}
    \includegraphics[width=\linewidth]{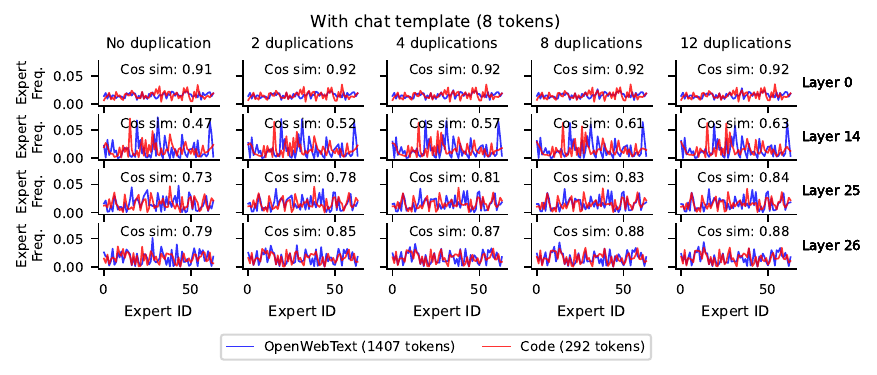}
    \caption{Comparison of expert usage under duplication for ERNIE-4.5-21B-A3B with (bottom) and without (top) chat template on two pieces of input.}
    \label{fig:ernie_duplication_with_chat_template}
\end{figure}

\begin{figure}[!h]
    \centering
    \includegraphics[width=0.6\linewidth]{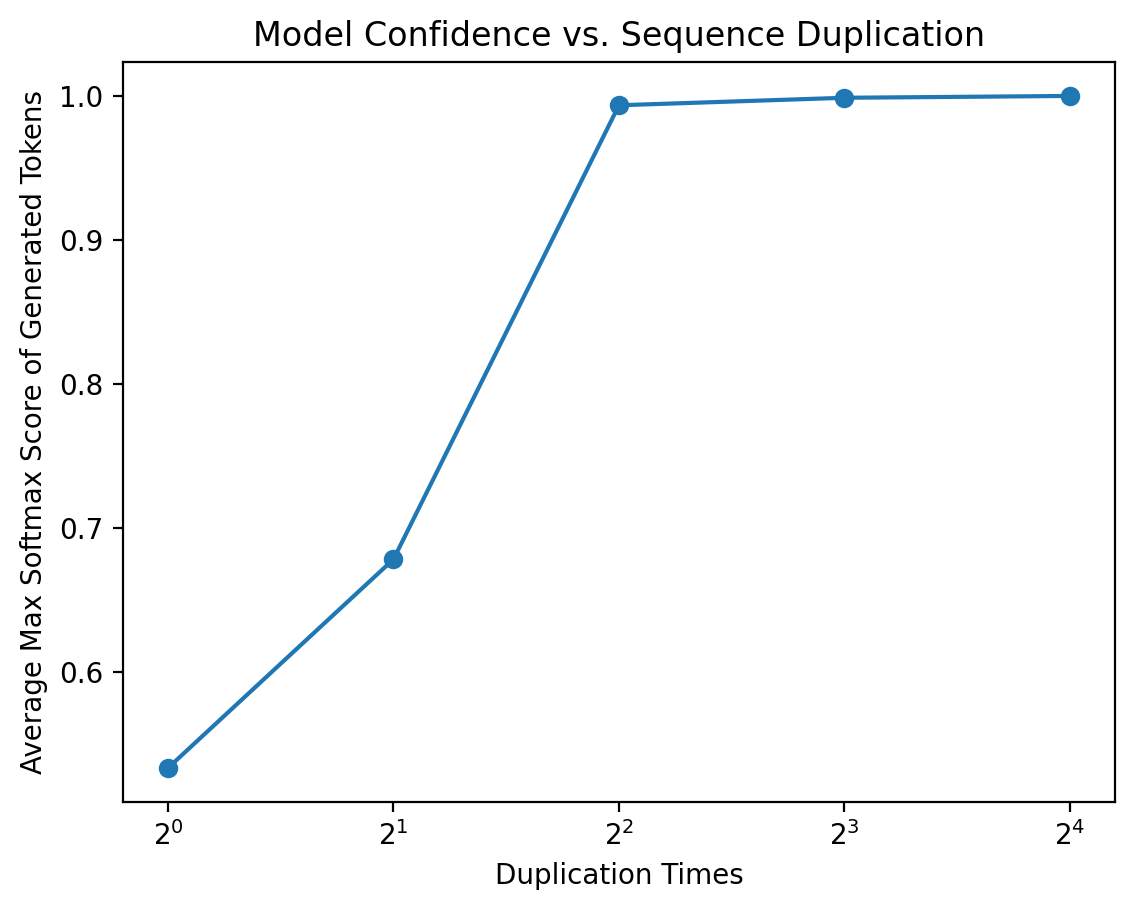}
    \caption{Duplication increases logit magnitude, increasing model confidence and causes random decoding to collapse.}
\end{figure}

% \section{Old results}

% \begin{figure}[!t]
% \captionsetup{justification=centering}
% \centering
%  \begin{subfigure}[t]{0.32\textwidth}
%     \centering
%      \includegraphics[height=3.6cm]{figures/same_token_different_context_gpt_oss_medical_vs_daily.pdf}
%      \caption{Same token,\\different context.}
%     \end{subfigure}
%     \begin{subfigure}[t]{0.32\textwidth}
%     \centering
%      \includegraphics[height=3.6cm]{figures/coffee_description_different_language.pdf}
%      \caption{Same paragraph,\\different language}
%     \end{subfigure}
%  \begin{subfigure}[t]{0.33\textwidth}
%     \centering
%      \includegraphics[height=3.5cm]{figures/different_token_same_meaning_gpt_oss_diffusion.pdf}
%      \caption{Same meaning,\\different expressions}
%     \end{subfigure}
%     \caption{Same sentence, different context \xw{Look at output} \xw{Look at last few tokens} Note that we expect the patterns here, especially ones in later layers, to be model dependent}
% \end{figure}

\end{document}